\documentclass[final,3p,times]{elsarticle}

\usepackage{amssymb}
\usepackage{amsthm}
\usepackage{array}
\usepackage{amsmath}
\usepackage{amssymb}
\usepackage{type1cm}
\usepackage{bm}
\usepackage{multirow}

\usepackage{makeidx}         
\usepackage{graphicx}        

\usepackage{multicol}        
\usepackage[bottom]{footmisc}
\usepackage{subfigure}
\usepackage{booktabs}
\usepackage{algorithm, algorithmic}
\usepackage{color}

\usepackage{hyperref}
\usepackage{geometry}
\usepackage{setspace}
\newtheorem{thm}{Theorem}

\newtheorem{observation}{Observation}

\DeclareMathOperator*{\argmin}{argmin}
\DeclareMathOperator*{\argmax}{argmax}

\journal{Pattern Recognition}

\begin{document}

\begin{frontmatter}



\title{Global and Local Structure Preserving Sparse Subspace Learning: An Iterative Approach to Unsupervised Feature Selection}


\author[uestc,ua]{Nan Zhou}
\ead{{nzhouuestc@gmail.com}}
\author[uw]{Yangyang Xu}
\ead{{yangyang.xu@uwaterloo.ca}}
\author[uestc]{Hong Cheng\corref{cor1}}
\ead{{hcheng@uestc.edu.cn}}
\author[uestc]{Jun Fang}
\ead{{JunFang@uestc.edu.cn}}
\author[ua]{Witold Pedrycz}
\ead{{wpedrycz@ualberta.ca}}

\address[uestc]{Center for Robotics, School of Automation Engineering, University of Electronic Science and Technology of China, Chengdu, Sichuan, 611731, China}
\address[ua]{Department of Electrical and Computer Engineering, University of Alberta, AB, T6G2R3, Canada}
\address[uw]{Department of Combinatorics and Optimization, University of Waterloo, ON, N2L3G1, Canada}

\begin{abstract}
As we aim at alleviating the curse of high-dimensionality, subspace learning is becoming more popular. Existing approaches use either information about global or local structure of the data, and few studies simultaneously focus on global and local structures as the both of them contain important information. In this paper, we propose a global and local structure preserving sparse subspace learning (GLoSS) model for unsupervised feature selection. The model can simultaneously realize feature selection and subspace learning. 
In addition, we develop a greedy algorithm to establish a generic combinatorial model, and an iterative strategy based on an accelerated block coordinate descent is used to solve the GLoSS problem. We also provide whole iterate sequence convergence analysis of the proposed iterative algorithm. Extensive experiments are conducted on real-world datasets to show the superiority of the proposed approach over several state-of-the-art unsupervised feature selection approaches.
\end{abstract}


\end{frontmatter}


\section{Introduction}
With the advances in data processing, the dimensionality of the data increases and can be extremely high in many fields such as computer vision, machine learning and image processing. The high dimensionality of the data not only greatly increases the time and storage space required to realize data analysis but also introduces much redundancy and noise which can decrease the accuracy of ensuing methods. Hence, dimensionality reduction becomes an important and often necessary preprocessing step to accomplish certain machine learning tasks such as clustering and classification.

Generally speaking, dimensionality reduction approaches can be divided into two classes: feature selection and subspace learning. Feature selection methods aim to select a subset of most representative features following a certain criterion (e.g.,\cite{guyon2003introduction,yang2004essence,herman2013mutual,yan2015sparse, zhao2013similarity}) , while subspace learning methods aim to learn a (linear or nonlinear) transformation to map the original high-dimensional data into a lower-dimensional subspace (e.g., \cite{wang2014robust,duda2012pattern,he2005neighborhood,cai2007spectral}). Subspace learning methods, such as principal component analysis (PCA), combine all original features at each dimension of the learned subspace, and this causes some interpretation difficulties. To overcome this difficulty, sparse subspace learning methods (e.g., \cite{zou2006sparse, moghaddam2006generalized, cai2007spectralsparse}) and joint models that simultaneously perform subspace learning and feature selection (e.g., \cite{cai2010unsupervised, gu2011joint, li2012unsupervised, wang2015subspace}) have been developed. 

This paper exhibits the following main contributions: 
\begin{enumerate}
\item We propose a novel \emph{unsupervised sparse subspace learning} model for feature selection. The model simultaneously perserves global and local structures of the data, both of which contain important discriminative information for feature selection, as demonstrated in \cite{du2013local, liu2014global}.
 We derive the model by first relaxing an existing combinatorial model and then adding a group sparsity regularization term. The regularization term controls the row sparsity of the transformation matrix, and since each row of the transformation matrix corresponds to a feature, the proposed model can automatically select representative features and makes easy interpretation. 
\item We, for the first time, propose a greedy algorithm to the original combinatorial optimization problem. In addition, we apply the accelerated block coordinate descent (BCD) method proposed in \cite{xu2013block} to the relaxed continuous but nonconvex problem. The BCD method utilizes the bi-convexity structure of the problem and has been found very efficient for our purposes. 
\item We establish a whole iterate sequence convergence result of the BCD method for our problem under consideration by assuming the existence of a full rank limit point. Because of the peculiarity of the formulated problem, the result is new and not implied by any existing convergence results of BCD. 
\item We conduct extensive experimental studies. The proposed method is tested on six real-world datasets coming from different areas and compared to eight state-of-the-art unsupervised feature selection algorithms. The results demonstrate the superiority of the proposed method over all the other compared methods. In addition, we study the sensitivity of the proposed method to the parameters of the model and observe that it can perform in a stable way within a large range of values of the parameters. 
\end{enumerate}

\subsubsection*{Organization and notation}
The paper is organized as follows. In Sect. \ref{sec.reworks}, we give a brief review of recent related studies on subspace learning. Sect. \ref{sec.LPSSLmodel} reviews two local structure preserving methods and proposes a local structure preserving sparse subspace learning model. In Sect. \ref{sec.algorihm}, we present an algorithm leading to the solution of the proposed model. Convergence results are also shown. Experimental results are reported in Sect. \ref{sec.experiment}. Finally, Sec. \ref{sec.conclusion} concludes this paper.

To facilitate the presentation of the material, we list a notation in Table \ref{tab.notation}.
\begin{table}[htbp]
\centering
\caption{Notation}\label{tab.notation}
\begin{tabular}{ ll }
\toprule
Notation & Description\\
\hline
$n\quad$ & The number of instances\\
$d\quad$ & The number of features\\
$\mathcal{\kappa}\quad$ & The number of selected features\\
$A_{i.}$ & The $i^{-th}$ row of the matrix $A$\\
$K\quad$ & The dimension of subspace\\
$m$ & The number of nearest neighbors\\
$\|W\|_{2,1}\quad$ & $\sum_i \|W_{i.}\|_{2}$, the sum of the $\ell_2$-norm of rows in $W$\\
$\|\mathbf{x}\|_0$ & $\sharp\{x_i\neq 0\}$, the number of nonzero elements in vector $\mathbf{x}$\\
$|\mathcal{I}|$ & cardinality of set $\mathcal{I}$\\
\toprule
\end{tabular}
\end{table} 

\section{Related Studies}\label{sec.reworks}
\subsubsection*{Subspace learning}

One well-known subspace learning method is principal component analysis (PCA) \cite{duda2012pattern,lipovetsky2009pca}. It maximizes the global data structure information in the principal space and thus it becomes optimal in terms of data fitting. Beside global structure, local structure of the data also contains important discriminative information \cite{bottou1992local}, which plays a crucial role in pattern recognition \cite{jiang2011linear}. Many subspace learning methods preserve different local structures of the data for different problems and can yield better performance than the traditional PCA method. These methods usually use the linear extension of graph embedding (LGE) to preserve local structure. With different choices of the graph adjoint matrix, LGE framework leads to different subspace learning methods. The popular ones include Linear Discriminant Analysis (LDA) \cite{duda2012pattern,ching2012regularized}, Locality Preserving Projection (LPP) \cite{he2005laplacian,niyogi2004locality,zhang2010graph} and Neighborhood Preserving Embedding (NPE) \cite{he2005neighborhood}.
One drawback of these locality preservation methods is that they require eigen-decomposition of dense matrices, which can be very expensive in both CPU time and machine storage, especially for problems involving high-dimensional data. To overcome this drawback, Cai \textit{et al.} \cite{cai2007spectral} proposed a Spectral Regression (SR) method to transform the eigen-decomposition problem into a two-step regression problem that becomes easier to solve. 


\subsubsection*{Sparse subspace learning}
Although the subspace learning can transform the original high-dimensional data into a lower-dimensional space, it mingles all features and lacks interpretability. For better interpretability, sparse subspace learning methods have been proposed in the literature by adding certain sparsity regularization terms or sparsity constraints into subspace learning models.
For example, the sparse PCA (SPCA) \cite{zou2006sparse} adds ``Elastic Net'' term into the traditional PCA. 
Moghaddam \textit{et al.} \cite{moghaddam2005spectral} proposed a spectral bounds framework for sparse subspace learning. 
Cai \textit{et al.} \cite{cai2007spectralsparse} proposed a unified sparse subspace learning method based on spectral regression model, which adds an $\ell_1$ regularization term in the regression step. 
Qiao \textit{et al.} \cite{qiao2010sparsity} introduced the Sparsity Preserving Projection (SPP) method for subspace learning, while SPP utilizes the sparsity coefficients to construct the graph Laplacian. 
It is worth mentioning that besides subspace learning, sparsity regularized methods have also been used in many other fields such as computer vision \cite{cheng2009sparsity,chengsparsity}, image processing \cite{fang2015pattern}, and signal recovery \cite{zhou2015bayesian}. 

\subsubsection*{Simultaneous feature selection and subspace learning}
Recently, joint methods have been proposed to simultaneously perform feature selection and subspace learning. The core idea of these methods is to use the transformation matrix to guide feature selection according to the norm of its row/column vectors. Cai \textit{et al.} \cite{cai2010unsupervised} combined the sparse subspace learning with feature selection and proposed the Multi-Cluster Feature Selection (MCFS) method. Because MCFS uses $\ell_1$-term to control the sparsity of the transformation matrix, different dimensions of the learned subspace may combine different features, and thus the model lacks sound interpretability. Gu \textit{et al.} \cite{gu2011joint} improved the MCFS method by using $\ell_{2,1}$-term to enforce the row sparsity of the transformation matrix. This way, the transformation matrix will have zero-rows corresponding to irrelavant features. 
Wang \textit{et al.} \cite{wang2015subspace} proposed an unsupervised feature selection framework, which uses the global regression term for subspace learning and orthogonal transformation matrix for feature selection. In general, the orthogonality constraint may limit its applications, as mentioned in \cite{qian2013robust} in practice, feature weight vectors are not necessarily orthogonal to each other. In addition, the model discussed in  \cite{wang2015subspace} does not utilize local structure of the data. As demonstrated in \cite{bottou1992local}, local structure of the data often contains essential discriminative information.

\subsubsection*{Other related works}
There are some other related methods for subspace learning. 
Provided with only weak label information (e.g., preference relationships between examples), 
Xu \textit{et al.} \cite{xu2014large} proposes  a Weakly Supervised Dimensionality Reduction (WSDR) method, which considers samples' pairwise angles and also distances. For the $K$-means problem, Boutsidis \textit{et al.} \cite{boutsidis2015randomized} proposed randomized feature selection and subspace learning methods and showed that a constant-factor approximation can be guaranteed with respect to the optimal $K$-means objective value. Other popular subspace learning methods include: Nonnegative Matrix Factorization (NMF) \cite{lee2001algorithms,yuan2005projective} that considers subspace learning of nonnegative data; 
joint LDA and $K$-means \cite{ding2007adaptive} that 
combines LDA and $K$-means clustering together for unsupervised subspace learning; Dictionary Learning (DL) \cite{zhang2011linear} that first learns a dictionary via sparse coding and then uses the dictionary to decompose each sample into more discriminative and less discriminative parts for subspace learning. For more subspace learning methods, see \cite{de2003framework} and the references therein.

\section{The Proposed Framework of Local Structure Preserving Sparse Subspace Learning}\label{sec.LPSSLmodel}

In this section, we introduce our feature selection models that encourage global data fitting and also preserve local structure information of the data. The first model is of combinatorial nature, only allowing 0-1 valued variables. The modeling idea is intuitive and inspired from (11) of \cite{wang2015subspace}, but it is not easy to find a good approximate solution to the problem. The second model relaxes the first one and becomes its continuous counterpart. Various optimization methods can be utilized to determine its solution. More importantly, we find that the relaxed model can most times produce better performance than the original one; one can refer to the numerical results reported in Section \ref{sec.experiment}. We want to emphasize again here that our main contributions concern the second model and the algorithm developed for it.

\subsection{A Generic Formulation}\label{sec:basic}
Given $n$ data samples $\{\mathbf{p}_i\}_{i=1}^n$ located in the $d$-dimensional space, the goal of feature selection is to find a small set of features that can capture most useful information of the data which can better serve to solve classification or clustering problems. 
One natural way to measure the information content is to see how close the original data samples are to the learned subspace spanned by the selected features. 
Mathematically, the distance of a vector $\mathbf{x}$ to a subspace $\mathcal{X}$ can be represented as $\|\mathbf{x}-\mathcal{P}_{\mathcal{X}}(\mathbf{x})\|_2$, where $\mathcal{P}_\mathcal{X}$ denotes the projection onto $\mathcal{X}$ and $\|\cdot\|_2$ is the Euclidean 2-norm. Hence, the feature selection problem can be described as follows
\begin{equation}\label{eq:orig-model}
  \begin{split}
  &\min_{W,H} \frac{1}{2}\|X-XWH\|_F^2\\
  &\ \text{s.t.}\quad W\in \{0,1\}^{d\times \kappa},\ W^{\top}\bm{1}_{d\times 1} = \bm{1}_{\kappa\times 1},\\
   &\ \qquad \|W \bm{1}_{\kappa\times 1}\|_{0}=\kappa.
  \end{split}
\end{equation}
where $X=[\mathbf{p}_1,\mathbf{p}_2,\ldots,\mathbf{p}_n]^\top\in \mathbb{R}^{n\times d}$.
Concerning the proposed model, we make a few remarks:
\begin{enumerate}
\item The matrix $W$ is the selection matrix with entries of ``$0$'' or ``1''. The constraint $W^\top\bm{1}_{d\times 1} = \bm{1}_{\kappa\times 1}$ enforces that each column of $W$ has only one ``$1$''. Therefore, at most $\kappa$ features are selected.
\item The constraint $\|W \bm{1}_{\kappa\times 1}\|_{0}=\kappa$ enforces that $W$ has $\kappa$ nonzero rows. No feature will be selected more than once, and thus exactly $\kappa$ features will be chosen.
\item Given $W$, the optimal $H$ produces the coefficients of all $d$ features projected onto the subspace spanned by the selected features. Hence, \eqref{eq:orig-model} expresses the distance of $X$ to the learned subspace.
\end{enumerate}

The recent work \cite{wang2015subspace} mentions to use the 0-1 feature selection matrix, but it does not explicitly formulate an optimization model like \eqref{eq:orig-model}. 
As shown in \cite{bottou1992local}, local structure of the data often contains discriminative information that is important for distinguishing different samples. To make the learned subspace preserve local structure, one can add a regularization term to the objective to promote such structural information, namely, to solve the regularized model
\begin{equation}\label{eqn.basicform}
  \begin{split}
  &\min_{W,H} \frac{1}{2}\|X-XWH\|_F^2+\mu \text{Loc}(W)\\
  &\ \text{s.t.}\quad W\in \{0,1\}^{d\times \kappa},\ W^\top\bm{1}_{d\times 1} = \bm{1}_{\kappa\times 1},\\
   &\ \qquad \|W \bm{1}_{\kappa\times 1}\|_{\ell_0}=\kappa,
  \end{split}
\end{equation}
where $\text{Loc}(W)$ is a local structure promoting regularization term, and $\mu$ is a parameter to balance the data fitting and regularization. In the next subsection, we introduce different forms of $\text{Loc}(W)$.

\subsection{Local Structure Preserving Methods}
Local structure of the data often contains important information that can be used to distinguish the samples \cite{cai2010unsupervised,he2005laplacian}. A predictor utilizing local structure information can be much more efficient than that only using global information \cite{bottou1992local}. Therefore, one may want the learned lower dimensional subspace to be able to preserve local structure of the training data. We briefly review two widely used local structure preserving methods.


\subsubsection{Local Linear Embedding}
The Local Linear Embedding (LLE) \cite{roweis2000nonlinear} method first finds the set $\mathcal{N}_m(\mathbf{p}_j)$ of $m$ nearest neighbors for all $j$ and then constructs the \emph{similarity matrix} $S$ as the (normalized) solution of the following problem 
\begin{equation}\label{eqn.LLE}
  \begin{split}
  \min_S &\ \sum_{i=1}^n \|\mathbf{p}_i-\sum_{j=1}^n S_{ij}\mathbf{p}_j\|_{2}^2,\\
  \text{s.t.} &\ S_{ij}=0,\,\forall j\not\in \mathcal{N}_m(\mathbf{p}_i),\,\forall i.
  \end{split}
\end{equation}
One can regard $S_{ij}$ as the coefficient of the $j^{-th}$ sample when approximating the $i^{-th}$ sample, and the coefficient is \emph{zero} if the $j^{-th}$ sample is not the neighbor of the $i^{-th}$ one.  After obtaining $S$ from \eqref{eqn.LLE}, LLE further normalizes it such that $\sum_{j=1}^nS_{ij}=1$. Then it computes the lower-dimensional representation $Y=W^\top X^\top\in\mathbb{R}^{K \times n}$ through solving the following problem 
\begin{equation}\label{eqn.LLE2}
  \begin{split}
\min_W\  \sum_{i=1}^n\|W^\top\mathbf{p}_i-\sum_{j=1}^nS_{ij}W^\top\mathbf{p}_j\|_{2}^2.
  \end{split}
\end{equation}
Note that if $W$ is a selection matrix defined as (\ref{eqn.basicform}), $W^\top\mathbf{p}_j$ becomes a lower-dimensional sample, keeping the $K$ selected features by $W$ and removing all other features.
Let $L = (I-S)^\top(I-S)$, where $I$ is the $n\times n$ identity matrix. Then it is easy to see that
(\ref{eqn.LLE2}) can be equivalently expressed as
\begin{equation}\label{eq:lle-reg}
  \begin{split}
\min_W\  Tr(W^\top X^\top LXW).
  \end{split}
\end{equation}

\subsubsection{Linear Preserve Projection}\label{sec:lpp}
For the Linear Preserve Projection \cite{niyogi2004locality} (LPP) method, the similarity matrix $S$ is generated by
\begin{equation}
  \begin{split}
S_{ij}=\begin{cases}
\exp (\frac{\|\mathbf{p}_i-\mathbf{p}_j\|^2_{2}}{-2\sigma^2})\quad &\mathbf{p}_i\in \mathcal{N}_m(\mathbf{p}_j)\ \text{or}\ \mathbf{p}_j\in \mathcal{N}_m(\mathbf{p}_i)\\
0&\text{otherwise}
\end{cases},
  \end{split}
\end{equation}
where $\mathcal{N}_m(\mathbf{p}_i)$ is the set of $m$ nearest neighbors of $\mathbf{p}_i$. The LPP method requires the lower-dimensional representation to preserve the similarity of the original data and forms the transformation matrix $W$ by solving the following optimization problem
\begin{equation}\label{eq:lpp1}
  \begin{split}
\min_W\ \sum_{i,j=1}^n S_{ij}\|W^\top\mathbf{p}_i - W^\top \mathbf{p}_j\|_{2}^2.
  \end{split}
\end{equation}
Let $L = D-S$ be the Laplacian matrix, where $D$ is a diagonal matrix, called degree matrix, with diagonal elements $D_{ii}=\sum_{j=1}^n S_{ij},\,\forall i$. Then \eqref{eq:lpp1} can be equivalently expressed as
\begin{equation}\label{eq:lpp-reg}
  \begin{split}
\min_W\ Tr(W^\top X^\top L XW).
  \end{split}
\end{equation}

\subsection{Relaxed Formulation}
The problem \eqref{eqn.basicform} is of combinatorial nature, and we do not have many choices to solve it. In the next section, we develop a greedy algorithm, which chooses $\kappa$ features one by one, with each selection decreasing the objective value the most among all the remaining features. Numerically, we observe that the greedy method can often make satisfactory performance. However, it can sometimes perform very bad; see results on Yale64 and Usps in section \ref{sec.experiment}.  For this reason, we seek an alternative way to select features by first relaxing \eqref{eqn.basicform} to a continuous problem and then employing a reliable optimization method to solve the relaxed problem. As observed in our tests, the relaxed method can perform comparably well with and, most of the time, much better than the original one.


As remarked at the end of Section \ref{sec:basic}, any feasible solution $W$ is nonnegative and has $\kappa$ non-zero rows. If $\kappa\ll d$ (that is usually satisfied), then $W$ has lots of zero rows. Based on these observations, we relax the 0-1 constraint to nonnegativity constraint and the hard constraints $W^T\bm{1}_{d\times 1} = \bm{1}_{\kappa\times 1}, \|W \bm{1}_{\kappa\times 1}\|_{\ell_0}=\kappa$ to $g(W)\le \kappa$, where $g(W)$ measures the row-sparsity of $W$. One choice of $g(W)$ is group Lasso \cite{yuan2006model}, i.e., \begin{equation}\label{eq:gw}
g(W)=\sum_{i=1}^d \|W_{i.}\|_2,
\end{equation}
where $W_{i.}$ denotes the $i$-th row of $W$.
This way, we relax \eqref{eqn.basicform} to
\begin{equation}\label{eqn.L1form}
  \begin{split}
  &\min_{W,H} \frac{1}{2}\|X-XWH\|_F^2+\mu\text{Loc}(W)\\
  &\ \text{s.t.}\quad W\in\mathbb{R}_+^{d\times K},\  g(W)\leq \kappa,
  \end{split}
\end{equation}
or equivalently
\begin{equation}\label{eqn.LPSSL}
  \begin{split}
  &\min_{W,H} \frac{1}{2}\|X-XWH\|_F^2+\mu\text{Loc}(W)+\beta g(W)\\
  &\ \text{s.t.}\quad W\in\mathbb{R}_+^{d\times K},
  \end{split}
\end{equation}
where $\mathbb{R}_+^{d\times K}$ denotes the set of $d\times K$ nonnegative matrices, and $\beta$ is a parameter corresponding to $\kappa$. Note that $W$ now also serves as a transformation matrix of subspace learning, and $K$ is the dimension of the learned subspace. It is not necessary $K=\kappa$. For better approximation by subspace learning, we will choose $K\ge \kappa$. We will focus on \eqref{eqn.LPSSL} because it is easier than \eqref{eqn.L1form} to solve. Practically, one needs to tune the parameters $\mu$, $\beta$, $\kappa$, and $K$. As shown in section \ref{sec.experiment}, the model with a wide range of values of the parameters can give stably satisfactory performance. 

Our model is similar to the Matrix Factorization Feature Selection (MFFS) model proposed in \cite{wang2015subspace}. The difference is that the MFFS model restricts the matrix $W$ to be orthogonal while we use regularization term $g(W)$ to promote row-sparsity of $W$. Although orthogonal $W$ makes their model closer to the original model \eqref{eq:orig-model}, it increases difficulty of solving their problem. In addition, MFFS does not utilize local structure preserving term as we do and thus may lose some important local information. Numerical tests in section \ref{sec.experiment} demonstrate that the proposed model along with an iterative method can produce better results than those obtained by using the MFFS method. 

Before completing this section, let us make some remarks on the relaxed model. Originally, $W$ is restricted to have exactly $\kappa$ non-zeros, so it could be extremely sparse as $\kappa\ll d$, and one may consider to include a sparsity-promoting term (e.g., $\ell_1$-norm) in the objective of \eqref{eqn.LPSSL}. However, doing so is not necessary since both $g(W)$ and the nonnegativity constraint encourage sparsity of $W$, and numerically we notice that $W$ output by our algorithm is indeed very sparse. Another point worth mentioning is that the elements of $W$ given by \eqref{eqn.LPSSL} are real numbers and do not automatically select $\kappa$ features. For the purpose of feature selection, after obtaining a solution $W$, we choose the features corresponding to the $\kappa$ rows of $W$ that have the largest norms because larger values imply more important roles played by the features.

\subsection{Extensions}
In \eqref{eqn.LPSSL}, Frobenius norm is used to measure the data fitting and typically suitable when Gaussian noise is assumed in the data and also commonly used if no priori information is assumed. One can of course use other norm or metric if different priori information is known. For instance, if the data come with outliers, one can employ the Cauchy Regression (CR) \cite{liu2014robustness} instead of the Frobenius norm to improve robustness and modify \eqref{eqn.LPSSL} read as 
\begin{equation}\label{eqn.CRSSL}
  \begin{split}
  &\min_{W,H} \sum_{j=1}^d\sum_{i=1}^n \ln\left[1+\left(\frac{X_{ij}-X_{i.}WH_{.j}}{\sigma}\right)^2\right]+\mu\text{Loc}(W)+\beta g(W)\\
  &\ \text{s.t.}\quad W\geq 0.
  \end{split}
\end{equation}
When the data involves heavy tailed noise, \cite{liu2015performance,guan2012mahnmf} suggest to use the Manhattan distance defined by $\|A\|_M = \sum_{i=1}^n\sum_{j=1}^m |A_{ij}|$, and this way, \eqref{eqn.LPSSL} can be modified to 
\begin{equation}\label{eqn.CRSSL}
  \begin{split}
  &\min_{W,H} \|X-XWH\|_M+\mu\text{Loc}(W)+\beta g(W)\\
  &\ \text{s.t.}\quad W\geq 0.
  \end{split}
\end{equation}

\section{Solving the Proposed Sparse Subspace Learning}\label{sec.algorihm}

In this section, we present algorithms to approximately solve \eqref{eqn.basicform} and \eqref{eqn.LPSSL}. Throughout the rest of the paper, we assume that $\text{Loc}(W)$ takes the function either as \eqref{eq:lle-reg} or \eqref{eq:lpp-reg} and $g(W)$ is given by \eqref{eq:gw}. Due to the combinatorial nature of \eqref{eqn.basicform}, we propose a greedy method to solve it. The problem \eqref{eqn.LPSSL} is smooth, and various optimization methods can be applied. Although its objective is nonconvex jointly with respect to $W$ and $H$, it is convex with regard to one of them while the other one is fixed. Based on this property, we choose the block coordinate descent method to solve \eqref{eqn.LPSSL}.


\subsection{Greedy Strategy for \eqref{eqn.basicform}}

In this subsection, a greedy algorithm is developed for selecting $\kappa$ out of $d$ features based on \eqref{eqn.basicform}. The idea is as follows: each time, we select one from the remaining unselected features such that the objective value is decreased the most. We begin the design of the algorithm  by making the following observation.


\begin{observation}
Let $\mathcal{I}_1$ and $\mathcal{I}_2$ be two index sets of features. Assume $\mathcal{I}_1\subseteq \mathcal{I}_2$, and $X_{\mathcal{I}_1}$ and $X_{\mathcal{I}_2}$ are submatrices of $X$ with columns indexed by  $\mathcal{I}_1$ and $\mathcal{I}_2$ respectively. Then
\begin{equation}\label{eqn.L1model}
  \begin{split}
\min_{H_1}\ \|X-X_{\mathcal{I}_1}H_1\|_F^2\geq \min_{H_2}\ \|X-X_{\mathcal{I}_2}H_2\|_F^2.
  \end{split}
\end{equation}
\end{observation}

From the above observation, if the current index set of selected features is $\mathcal{I}$, the data fitting will become no worse if we enlarge $\mathcal{I}$ by adding more features. Below we describe in details on how to choose such additional features. Assume $X$ is normalized such that
\begin{equation}\label{eq:xnml}
\|\mathbf{x}_j\|_{2}=1,\,j=1,\ldots,d,
\end{equation} where $\mathbf{x}_j$ denotes the $j^{th}$ column of $X$. Let $\mathcal{I}$ be the current index set of selected features. The optimal $H$ to $\min_H\|X-X_{\mathcal{I}}H\|_F$ is given by
\begin{equation}\label{eq:optH}
H^*=(X_\mathcal{I}^\top X_\mathcal{I})^\dagger X_\mathcal{I}^\top X,
\end{equation}
where $^\dagger$ denotes the Moore-Penrose pseudoinverse of a matrix. Now consider to add one more feature into $\mathcal{I}$, say the $j^{th}$ one. Then the lowest data fitting error is
\begin{align*}&\,\min_{\mathbf{h}}\|X-X_\mathcal{I}H^*-\mathbf{x}_j\mathbf{h}\|_F^2\\
=&\, \min_{\mathbf{h}}\|\mathbf{h}\|_F^2-2\langle \mathbf{h}, \mathbf{x}_j^\top(X-X_\mathcal{I}H^*)\rangle +\|X-X_\mathcal{I}H^*\|_F^2\\
=&\, -\|\mathbf{x}_j^\top(X-X_\mathcal{I}H^*)\|_2^2+\|X-X_\mathcal{I}H^*\|_F^2,
\end{align*}
where the last equality is achieved at $\mathbf{h}=\mathbf{x}_j^\top(X-X_\mathcal{I}H^*)$. Hence, we can choose $j$ such that $\|\mathbf{x}_j^\top(X-X_\mathcal{I}H^*)\|_2$ is the largest among all features not in $\mathcal{I}$.


Carrying out a comparison to $\|\mathbf{x}_j^\top(X-X_\mathcal{I}H^*)\|_2$, we find that $\|\mathbf{x}_j^\top(X-X_\mathcal{I}H^*)\|_1$ can serve better. It turns out that the latter is exactly the correlation between $\mathbf{x}_j$ and the residual $X-X_\mathcal{I}H^*$. Denote the correlation between $\mathbf{x}_i$ and $X$ as
\begin{equation*}
  \begin{split}
Cor(\mathbf{x}_i,X) = \sum_{s=1}^d |\mathbf{x}_i^\top\mathbf{x}_s|.
  \end{split}
\end{equation*}
As shown in \cite{tropp2007signal}, if $Cor(\mathbf{x}_i,X)$ is large, then 
the columns of $X$ can be better linearly represented by $\mathbf{x}_i$. To preserve local structure, we need also incorporate $\text{Loc}(W)$. If the set of selected features is $\mathcal{I}$, then 
\begin{equation*}
  \begin{split}
\text{Loc}(W)=Tr(W^\top X^\top LXW) = \sum_{i\in I} \mathbf{x}_i^\top L\mathbf{x}_i.
  \end{split}
\end{equation*}
Assuming $L=D-S$, i.e., using the LPP method in section \ref{sec:lpp} (that is used throughout our tests), we have from \eqref{eq:xnml} that
$$\min_{j\not\in \mathcal{I}} \mathbf{x}_j^\top L \mathbf{x}_j \Leftrightarrow \max_{j\not\in \mathcal{I}} \mathbf{x}_j^\top S \mathbf{x}_j.$$
Therefore, we can enlarge $\mathcal{I}$ by adding one more feature index $j^*$ such that
$$j^*\in\argmax_{j\not\in \mathcal{I}}Cor(\mathbf{x}_j, X-X_\mathcal{I}H^*)+\mathbf{x}_j^\top S\mathbf{x}_j,$$
where $H^*$ is given in \eqref{eq:optH}, and we have set $\mu=1$ in \eqref{eqn.basicform} for simplicity. Algorithm \ref{alg:glpslfs} summarizes our greedy method, and for better balancing the correlation and local structure preserving terms, we normalize both of them in the 5th line of Algorithm \ref{alg:glpslfs}.

\begin{algorithm}\caption{\textbf{G}reedy \textbf{L}ocally \textbf{P}reserved \textbf{S}ubspace \textbf{L}earning (GLPSL)}\label{alg:glpslfs}
\begin{algorithmic}[1]
\STATE \textbf{Input}: Data matrix $X\in\mathbb{R}^{n\times d}$, and the number $\kappa$ of features to be selected.
\STATE \textbf{Output}: Index set of selected features $\mathcal{I}\subseteq \{1,\ldots,d\}$ with $|\mathcal{I}|=\kappa$.
\STATE \textbf{Initialize} residual $R=X$, candidate set $\Omega = \{1,2,\ldots,d\}$, selected set $\mathcal{I} = \emptyset$.
\FOR{$i=1$ to $\kappa$}
\STATE $i \leftarrow \arg\max_{i\in\Omega}  \frac{Cor(\mathbf{x}_i,R)}{\sum_{j\in\Omega} Cor(\mathbf{x}_j,R)} + \frac{\mathbf{x}_i^\top S\mathbf{x}_i}{\sum_{j\in\Omega} \mathbf{x}_j^\top S\bm{x}_i}$.
\STATE $\Omega \leftarrow \Omega\backslash \{i\}$ and $\mathcal{I} = \mathcal{I}\cup \{i\}$.
\STATE $R \leftarrow X- X_\mathcal{I}(X_\mathcal{I}^\top X_\mathcal{I})^{\dagger}X_\mathcal{I}^\top X.$
\ENDFOR
\end{algorithmic}
\end{algorithm}

\subsection{Accelerated block coordinate update method for \eqref{eqn.LPSSL}}

In this subsection, we present an alternative method for feature selection based on \eqref{eqn.LPSSL}. Utilizing bi-convexity of the objective, we employ the accelerated block coordinate update (BCU) method proposed in \cite{xu2013block} to solve \eqref{eqn.LPSSL}. As explained in \cite{xu2013block}, BCU especially fits to solving bi-convex\footnote{More precisely, in \cite{xu2013block}, BCU is proposed to solve multi-convex optimization problems, which includes bi-convex problems as special cases.} optimization problems like \eqref{eqn.LPSSL}. It owns low iteration-complexity as shown in section \ref{sec.complexity} and also guarantees the whole iterate sequence convergence on solving \eqref{eqn.LPSSL} as shown in section \ref{sec.convergence}. The whole iterate sequence convergence is important because otherwise running the algorithm for different numbers of iterations may result in significantly different solutions, which will further affect the clustering or classfication results. Many existing methods such as the multiplicative rule method \cite{lee2001algorithms} only guarantee nonincreasing monotonicity of the objective values or iterate subsequence convergence, and thus our convergence result is much stronger.


Following the framework of BCU, our algorithm is derived by alternatingly updating $W$ and $H$, one at a time while the other one is fixed at its most recent value. Specifically, let
\begin{align}
&f(W,H)=\frac{1}{2}\|X-XWH\|_F^2+ \frac{\mu}{2}Tr(W^\top X^\top L XW),\label{eq:fun-f}\\
&g_\beta(W)=\beta\|W\|_{2,1}.\label{eq:fun-g}
\end{align}
At the $k$-th iteration, we perform the following updates:
\begin{subequations}
\begin{align}
&W^{k+1}=\argmin_{W\ge 0} \langle \nabla_W f(\hat{W}^k,H^k), W-\hat{W}^k\rangle+\frac{L_w^k}{2}\|W-\hat{W}^k\|_F^2+ g_\beta(W),\label{eq:update-w}\\
&H^{k+1}=\argmin_{H} f(W^{k+1},H),\label{eq:update-h}
\end{align}
\end{subequations}
where we take $L_w^k$ as the Lipschitz constant of $\nabla_W f(W,H^k)$ with respect to $W$ and
\begin{equation}\label{eq:ex-w}
\hat{W}^k=W^k+\omega_k(W^k-W^{k-1})
\end{equation}
is an extrapolated point with weight $\omega_k\in[0,1],\,\forall k$.

Note that the $H$-subproblem \eqref{eq:update-h} can be simply reduced to a linear equation and has the closed-form solution:
\begin{equation}\label{eq:sol-h}
H^{k+1} = \left[(W^{k+1})^\top X^\top X(W^{k+1})\right]^{\dagger}(W^{k+1})^\top X^\top X.
\end{equation}
If $H$ is restricted to be nonnegative, in general, \eqref{eq:update-h} does not exhibit a closed-form solution. In this case, one can update $H$ in the same manner as that of $W$, i.e., completing a block proximal-linearization update. In the following, we discuss in details on parameter settings and how to solve $W$-subproblem \eqref{eq:update-w}.

%

\subsubsection{Parameter settings}
By direct computation, it is not difficult to have
\begin{equation}\label{grad-W}\nabla_W f(W,H)=X^\top(XWH-X)H^\top + \mu X^\top LXW.
\end{equation}
For any $\hat{W},\tilde{W}$, we have
\begin{align*}
&\,\|\nabla_W f(\hat{W},H)-\nabla_W f(\tilde{W},H)\|_F\\
=&\,\|X^\top(X\hat{W}H-X)H^\top + \mu X^\top LX\hat{W}-X^\top(X\tilde{W}H-X)H^\top - \mu X^\top LX\tilde{W}\|_F\\
\le &\, \|X^\top (X\hat{W}H-X)H^\top-X^\top(X\tilde{W}H-X)H^\top\|_F+\|\mu X^\top LX\hat{W}-\mu X^\top LX\tilde{W}\|_F\\
= & \,\|X^\top X(\hat{W}-\tilde{W})HH^\top\|_F+\mu \|X^\top LX(\hat{W}-\tilde{W})\|_F\\
\le &\, \left(\|X^\top X\|_2\|HH^\top\|_2+\mu\|X^\top LX\|_2\right)\|\hat{W}-\tilde{W}\|_F,
\end{align*}
where $\|A\|_2$ denotes the spectral norm and equals the largest singular value of $A$, the first inequality follows from the triangle inequality, and the last inequality is from the fact $\|AB\|_F\le\|A\|_2\|B\|_F$ for any matrices $A$ and $B$ of appropriate sizes. Hence, $\|X^\top X\|_2\|HH^\top\|_2+\mu\|X^\top LX\|_2$ is a Lipschitz constant of $\nabla_W f(W,H)$ with respect to $W$, and in \eqref{eq:update-w}, we set
\begin{equation}
L_w^{k} = \|H^{k}(H^{k})^\top\|_2\|X^\top X\|_2+\mu\|X^\top LX\|_2. \label{eqn.LW}
\end{equation}

As suggested in \cite{xu2013block}, we set the extrapolation weight as
\begin{align}
\omega_k = \min\left(\hat{\omega}_{k}, \delta_{\omega} \sqrt{\frac{L_w^{k-1}}{L_w^{k}}}\right), \label{eqn.Wextraweight}
\end{align}
where $\delta_{\omega}<1$ is predetermined and $\hat{\omega}_{k}= \frac{t_{k-1}-1}{t_k}$ with
$$
t_0=1,\quad t_k=\frac{1}{2}\left(1+\sqrt{1+4t_{k-1}^2}\right).
$$
The weight $\hat{w}_k$ has been used to accelerate proximal gradient method for convex optimization problem (cf. \cite{beck2009fast}). It is demonstrated in \cite{xu2014ecyclic, xu2015apg-ntd} that the extrapolation weight in \eqref{eqn.Wextraweight} can significantly accelerate BCU for nonconvex problems.

\begin{algorithm}\caption{Proximal operator for nonnegative group Lasso: $W=$ Prox-NGL($Y,\lambda$)}\label{alg:prox-ngl}
\begin{algorithmic}
\FOR{$i = 1,\ldots,d$}
\STATE Let $\mathbf{y}$ be the $i^{th}$ row of $Y$ and $\mathcal{I}$ the index set of positive components of $\mathbf{y}$
\STATE Set $\mathbf{w}$ to a zero vector
\IF{$\|\mathbf{y}_\mathcal{I}\|_2>\lambda$}
\STATE Let $\mathbf{w}_\mathcal{I}=(\|\mathbf{y}_\mathcal{I}\|_2-\lambda)\frac{\mathbf{y}_\mathcal{I}}{\|\mathbf{y}_I\|_2}$
\ENDIF
\STATE Set the $i^{th}$ row of $W$ to $\mathbf{w}$
\ENDFOR
\end{algorithmic}
\end{algorithm}

\subsubsection{Solution of $W$-subproblem}

Note that \eqref{eq:update-w} can be equivalently written as
$$\min_{W\ge 0}\frac{1}{2}\left\|W-\left(\hat{W}^k-\frac{1}{L_w^k}\nabla_W f(\hat{W}^k,H^k)\right)\right\|_F^2+\frac{1}{L_w^k}g_\beta(W),$$
which can be decomposed into $d$ smaller independent problems, each one involving one row of $W$ and coming in the form
\begin{equation}\label{eq:small-prob}
\min_{\mathbf{x}\ge 0}\frac{1}{2}\|\mathbf{x}-\mathbf{y}\|_2^2+\lambda\|\mathbf{x}\|_2.
\end{equation}
We show that \eqref{eq:small-prob} has a closed-form solution and thus \eqref{eq:update-w} can be solved explicitly.

\begin{thm}\label{thm:nonneg-l2}
Given $\mathbf{y}$, let $\mathcal{I}=\{i:\,y_i>0\}$ be the index set of positive components of $\mathbf{y}$. Then the solution $\mathbf{x}^*$ of \eqref{eq:small-prob} is given as follows
\begin{enumerate}
\item For any $i\not\in \mathcal{I}$, $x_i^*=0$;
\item If $\|\mathbf{y}_\mathcal{I}\|_2\le \lambda$, then $\mathbf{x}^*_\mathcal{I}=0$; otherwise,
$\mathbf{x}^*_\mathcal{I}=(\|\mathbf{y}_\mathcal{I}\|_2-\lambda)\frac{\mathbf{y}_\mathcal{I}}{\|\mathbf{y}_\mathcal{I}\|_2}$.
\end{enumerate}
\end{thm}

\begin{proof}
For $i\not\in \mathcal{I}$, we must have $x_i^*=0$ because if $x_i^*>0$, setting the $i^{th}$ component to \emph{zero} and keeping all others unchanged will simultaneously decrease $(x_i-y_i)^2$ and $\|\mathbf{x}\|_2$. Hence, we can reduce \eqref{eq:small-prob} to the following form
\begin{equation}\label{eq:red-prob}
\min_{\mathbf{x}_\mathcal{I}\ge 0} \frac{1}{2}\|\mathbf{x}_\mathcal{I}-\mathbf{y}_\mathcal{I}\|_2^2+\lambda\|\mathbf{x}_\mathcal{I}\|_2.
\end{equation}
Without nonnegativity constraint on $\mathbf{x}_\mathcal{I}$, the minimizer of \eqref{eq:red-prob} is given by item 2 of Theorem \ref{thm:nonneg-l2} (for example, see \cite{parikh2013proximal}). Note that the given $\mathbf{x}_\mathcal{I}^*$ is nonnegative. Hence, it solves \eqref{eq:red-prob}, and this completes the proof.
\end{proof}

The above proof gives a way to find the solution of \eqref{eq:small-prob}. Using this method, we can explicitly form the solution of \eqref{eq:update-w} by the subroutine Prox-NGL in Algorithm \ref{alg:prox-ngl}, where $Y\in \mathbb{R}^{d\times K}$ and $\lambda>0$ are inputs, and $W$ is the output. Arranging the above discusstion together, we have the pseudocode in Algorithm \ref{alg:lpsslfs} for solving \eqref{eqn.LPSSL}.

\begin{algorithm}\caption{\textbf{G}lobal and \textbf{Lo}cal Structure Preserving \textbf{S}parse \textbf{S}ubspace Learning (GLoSS)}\label{alg:lpsslfs}
\begin{algorithmic}[1]
\STATE \textbf{Input}: Data matrix $X\in\mathbb{R}^{n\times d}$, the number of selected features $\kappa$ and parameter $\beta, \mu$.
\STATE \textbf{Output}: Index set of selected features $\mathcal{I}\subseteq \{1,\ldots,d\}$ with $|\mathcal{I}|=\kappa$
\STATE \textbf{Initialize} $W^0\in\mathbb{R}_+^{d\times K}$, $H^0\in\mathbb{R}^{K\times d}$, choose a positive number $\delta_{\omega}<1$; set $k=1$.
\WHILE{Not convergent}
\STATE Compute $L_w^{k}$ and $\omega_{k}$ according to (\ref{eqn.LW}) and (\ref{eqn.Wextraweight}) respectively.
\STATE Let $\hat{W}^{k}=W^{k}+\omega_{k}(W^{k}-W^{k-1})$.
\STATE Update $W^{k+1}\leftarrow\text{Prox-NGL}(\hat{W}^{k}-\frac{1}{L_w^k}\nabla_W f(\hat{W}^{k},H^k), \frac{\beta}{L_w^k})$.
\STATE Update $H^{k+1} \leftarrow \eqref{eq:sol-h}.$
\IF{$F(W^{k+1},H^{k+1})\geq F(W^{k},H^{k})$}
\STATE Set $\hat{W}^{k}=W^{k}$.
\ELSE
\STATE Let $k\leftarrow k+1$.
\ENDIF
\ENDWHILE
\STATE Normalize each column of $W$.
\STATE Sort $\|W_{i.}\|_{2},\ i=1,\ldots,d$ and select features corresponding to the $\kappa$ largest ones.
\end{algorithmic}
\end{algorithm}

\subsection{Complexity Analysis}\label{sec.complexity}
In this section, we count the flops per iteration of Algorithm \ref{alg:lpsslfs}. Our analysis is for general case, namely, we do not assume any structure of $X$. Note that if $X$ is sparse, the computational complexity will be lower. The main cost of our algorithm is in the update of $W$ and $H$, i.e., the $7^{th}$ and $8^{th}$ lines in Algorithm \ref{alg:lpsslfs}. For updating $W$, the major cost is in the computation of $\nabla_W f(W,H)$. Assume the dimension of subspace $K< \min(d,n)$. Then from \eqref{grad-W}, we can obtain the partial gradient by first computing $XW$, $HH^\top$ and $XH^\top$, then $XW(HH^\top)$ and $\mu L(XW)$, and finally left multiplying $X^\top$ to $XW(HH^\top)-XH^\top+\mu L(XW)$. This way, it takes about $3ndK+dK^2+nK^2+n^2K$ flops. To update $H$ by \eqref{eq:sol-h}, we can use the same trick and obtain $H$ in about $2ndK+nK^2+dK^2+K^3$ flops. Note that with $XW$ and $LXW$ pre-computed, the objective value required in $9^{th}$ line can be easily obtained in about $nd$ flops. Therefore, we have the per-iteration computational complexity of order $\mathcal{O}(ndK+n^2K)$ since $K< \min(d,n)$, and if $K=\mathcal{O}(1)$, then the algorithm is scalable to data size.



\subsection{Convergence analysis}\label{sec.convergence}

In this section, we analyze the convergence of Algorithm GLoSS. Let us denote $$\iota_+(W)=\left\{\begin{array}{ll}
0,&\text{ if }W\ge 0,\\
+\infty, &\text{ otherwise}
\end{array}\right.$$
to be the indicator function of the nonnegative quadrant. Also, let us denote
$$F(W,H)=f(W,H)+g_\beta(W)+\iota_+(W).$$
Then the problem \eqref{eqn.LPSSL} is equivalent to
$$\min_{W, H} F(W,H),$$
and the first-order optimality condition is $0\in\partial F(W,H)$. Here, $\partial F$ denotes the subdifferential of $F$ (see \cite{rockafellar2009variational} for example) and equals $\nabla F$ if $F$ is differentiable and a set otherwise. By Proposition 2.1 of \cite{attouch2010proximal}, $0\in\partial F(W,H)$ is equivalent to
$$0\in \partial_W F(W,H), \text{ and } 0=\nabla_H F(W,H)$$
namely,
\begin{subequations}\label{eq:opt-cond}
\begin{align}
&0\in\nabla_W f(W,H)+\partial g_\beta(W)+\partial\iota_+(W),\label{eq:opt-cond-w}\\
&0=\nabla_Hf(W,H).\label{eq:opt-cond-h}
\end{align}
\end{subequations}
We call $(W,H)$ a critical point of \eqref{eqn.LPSSL} if it satisfies \eqref{eq:opt-cond}.

In the following, we first establish a subsequence convergence result, stating that any limit point of the iterates is a critical point. Assuming existence of a full rank limit point, we further show that the whole iterate sequence converges to a critical point. The proofs of both results involve many technical details and thus are deferred to the appendix for the readers' convenience.

 \begin{thm}[Iterate subsequence convergence]\label{thm:subcvg}
Let $\{(W^k,H^k)\}_{k=1}^\infty$ be the sequence generated from Algorithm \ref{alg:lpsslfs}. Any finite limit point of $\{(W^k,H^k)\}_{k=1}^\infty$ is a critical point of \eqref{eqn.LPSSL}.
\end{thm}

Due to the coercivity of $g(W)$ and the nonincreasing monotonicity of the objective value, $\{W^k\}$ must be bounded. However, in general, we cannot guarantee the boundedness of $\{H^k\}$ because $XW^k$ may be rank-degenerate (i.e., not full rank). As shown in the next theorem, if we have rank-nondegeneracy of $XW^k$ in the limit, a stronger convergence result can be established. The nondegeneracy assumption is similar to that assumed in \cite[section 7.3.2]{GolubVanLoan1996} and \cite{convergence-HOOI} for (higher-order) orthogonal iteration methods.

\begin{thm}[Whole iterate sequence convergence]\label{thm:glbcvg}
Let $\{(W^k,H^k)\}_{k=1}^\infty$ be the sequence generated from Algorithm \ref{alg:lpsslfs}. If there is a finite limit point $(\bar{W},\bar{H})$ such that $X\bar{W}$ is full-rank, then the whole sequence $\{(W^k,H^k)\}_{k=1}^\infty$ must converge to $(\bar{W},\bar{H})$.
\end{thm} 

\section{Experimental Studies}\label{sec.experiment}
In this section, the proposed methods GLPSL (Algorithm \ref{alg:glpslfs}) and GLoSS (Algorithm \ref{alg:lpsslfs}) are tested on six benchmark datasets and compared to one widely used subspace learning method PCA and seven state-of-the-art unsupervised feature selection methods.

\subsection{Datasets}

The six benchmark datasets we use come from different areas, and their characteristics are listed in Table \ref{tab.datasets}. Yale64, WarpPIE, Orl64 and Orlraws\footnote{http://featureselection.asu.edu/datasets.php}  are face images, each sample of the datasets representing a face image. Usps\footnotemark[3] is a handwritten digit dataset that contains 9,298 handwritten digit images. 
Isolet\footnotemark[3]\footnotetext[3]{http://www.cad.zju.edu.cn/home/dengcai/Data/data.html} is a speech signal dataset containing 30 speakers' speech signal of alphabet twice. All datasets are normalized such that the vector corresponding to each feature has unit $\ell_2$-norm.
\begin{table}[htp]
\centering
\footnotesize
\caption{The datasets detail}
\label{tab.datasets}
\begin{tabular}{ccccc}
  \hline
  Dataset & $\sharp$ Instances & $\sharp$ Features & $\sharp$ Classes & Type of Data\\
  \hline
  Yale64 &165 &4096 & 15 & Face image\\
  WarpPIE & 210 & 2420 &10 & Face image\\
  Orl64 &400 &4096 &50 & Face image\\
  Orlraws &100 &10304 &10 & Face image\\
  Usps &9298 &256 &10 & Digit image\\
  Isolet &1560 &617 &26 & Speech signal\\
  \hline
\end{tabular}
\end{table}

\subsection{Experimental Settings}
Our algorithms are compared to the following methods:
\begin{enumerate}
\item \textit{PCA}: Principal component analysis (PCA) \cite{duda2012pattern} is an unsupervised subspace learning method that maximizes global structure information of the data in the principal space.
\item \textit{LS}: Laplacian score (LS) method \cite{he2005laplacian} uses the Laplacian score to evaluate effectiveness of the features. It selects the features individually that retain the samples' local similarity specified by a similarity matrix.
\item \textit{MCFS}: Multi-cluster feature selection (MCFS) \cite{cai2010unsupervised} is a two-step method, and it formulates the feature selection process as a spectral information regression problem with $\ell_{1}$-norm regularization term.
\item \textit{UDFS}: Unsupervised discriminative feature selection (UDFS) method  \cite{yang2011l2} combines the data's local discriminative property and the $\ell_{2,1}$-norm sparse constraint in one convex model to select the features which have the highest power of local discriminative property.
\item \textit{RSR}: Regularized self-representation (RSR) feature selection method \cite{zhu2015unsupervised} uses the $\ell_{2,1}$-norm to measure the fitting error and also $\ell_{2,1}$-norm to promote sparsity. Specifically, it solves the following problem:
    $$\min_W \|X-XW\|_{2,1}+\beta\|W\|_{2,1}.$$
\item \textit{NDFS}: Nonnegative Discriminative Feature Selection (NDFS) method \cite{li2012unsupervised} utilizes the nonnegative spectral analysis with $\ell_{2,1}$-norm regularization term.
\item \textit{GLSPFS}: Global and local structure preservation for feature selection (GLSPFS) method \cite{liu2014global} uses both global and local similarity structure to model the feature selection problem. It solves the following problem:
    $$\min_W \|V-XW\|_F^2+\mu Tr(W^\top X^\top LXW)+ \beta\|W\|_{2,1}$$
\item \textit{MFFS}: Matrix factorization feature selection (MFFS) method \cite{wang2015subspace} is similar to ours. It performs the subspace learning and feature selection process simultaneously by enforcing a nonnegative orthogonal transformation matrix $W$. This solves the following problem:
    \begin{equation}
      \begin{split}
      &\min_{W,H} \frac{1}{2}\|X-XWH\|_F^2\\
      &\ \text{s.t.}\quad W\geq 0,\ H\geq 0,\ W^\top W=I.
      \end{split}
    \end{equation}
\end{enumerate}

There are some parameters we need to set in advance. The dimension of the subspace is fixed to $K=100$ for GLoSS method, and the number of selected features $\kappa$ is taken from $\{20,30,40,50,60,70,80,90,100\}$ for all datasets. We use the LPP method in section \ref{sec:lpp} to preserve local structure of the data in GLSPFS, NDFS, GLPSL and GLoSS because both MCFS and LS use the LPP Laplacian graph, and we set the number of nearest neighbors to $m=5$ for LS, MCFS, UDFS, GLSPFS, NDFS, GLPSL and GLoSS. The parameter $m$ is required by LS, MCFS, GLSPFS, NDFS, GLPSL and GLoSS to build a similarity matrix and UDFS to build the local total scatter and between-class scatter matrices. 
For simplicity, the parameter of local structure preserving term is fixed to $\mu = 1$ in GLSPFS and GLoSS for all the tests in Sections \ref{sec:test1} and \ref{sec:test2}. We study the sensitivity of GLoSS to $\mu$ in Section \ref{sec:test3}.  
The sparsity parameter for UDFS, RSR, GLSPFS, NDFS and GLoSS is tuned from $\{0.01,0.1,1,10,40,70,100\}$. After completing the feature selection process, we use the $K$-means algorithm to cluster the samples using the selected features. The number of iterations of UDFS, GLSPFS, NDFS, MFFS, and GLoSS are set to 30. Because the performance of $K$-means depends on the initial point, we run it 20 times with different random starting points and report the average value. 

The compared algorithms are evaluated based on their clustering results. 
For each sample of all datasets, we set its class number as the cluster number. To measure the clustering performance, we use clustering accuracy (ACC) and normalized mutual information (NMI), which are defined below. 
Let $p_i$ and $q_i$ be the predicted and true labels of the $i^{-th}$ sample, respectively. The ACC is computed as 
\begin{equation}\label{eqn.EquTrace}
  \begin{split}
ACC = \frac{\sum_{i=1}^n \delta(q_i,map(p_i))}{n},
  \end{split}
\end{equation}
where $\delta(a,b)=1$ if $a=b$ and $\delta(a,b)=0$ otherwise, and $map(\cdot)$ is a permutation mapping that maps each predicted label to the equivalent true label. We use the Kuhn-Munkres algorithm \cite{laszlo2009matching} to realize such a mapping. High value of ACC indicates the predicted labels are close to the true ones, and thus the higher ACC is, the better the clustering result is.
The NMI is used to measure the similarity of two clustering results. For two label vectors $P$ and $Q$, it is defined as 
\begin{equation}\label{eqn.EquTrace}
  \begin{split}
NMI(P,Q)= \frac{I(P,Q)}{\sqrt{H(P)H(Q)}},
  \end{split}
\end{equation}
where $I(P,Q)$ is the mutual information of $P$ and $Q$, $H(P)$ and $H(Q)$ are the entropies of $P$ and $Q$ \cite{gray2011entropy}. In our experiments, $P$ contains the clustering labels using the selected features and $Q$ the true labels of samples in the dataset. Higher value of $NMI(P,Q)$ implies that $P$ better predicts $Q$.

\subsection{Experimental results}
In this subsection, we report the results of all tested methods. In addition, we study the sensitivity of the parameters present in \eqref{eqn.LPSSL}.

\subsubsection{Performance comparison}\label{sec:test1}
In Tables \ref{tab.ACC} and \ref{tab.NMI}, we present the ACC and NMI values produced by different methods. For each method, we vary the number of selected features among $\{20,30,40,\ldots,100\}$ and report the best result. 
From the tables, we see that GLoSS performs the best among all the compared methods except for Yale64 and WarpPIE in Table \ref{tab.ACC} and Yale64 and Orl64 in Table \ref{tab.NMI}, for each of which GLoSS is the second best. 
In addition, we see that the greedy method GLPSL performs reasonably well in many cases but can be very bad in some cases such as Usps in both Tables, and this justifies our reason to relax \eqref{eqn.basicform} and develop GLoSS method. Finally, we see that GLoSS outperforms MFFS for all datasets, and this is possibly due to the local structure preserving term used in GLoSS.
\begin{table}[htp]
\centering
\footnotesize
\caption{Clustering results (ACC\% $\pm$ std\%) of different feature selection algorithms on different datasets. The best results are
highlighted in \textbf{bold} and the second best results are underlined. (The higher ACC is, the better the result is.)}
\label{tab.ACC}
\begin{tabular}{ccccccc}
  \hline
  Dateset    & Isolet           & Yale64           & Orl64            & WarpPIE &        Usps               & Orlraw \\
  \hline
  PCA & 47.90 $\pm$ 2.97 & 32.79 $\pm$ 3.22 & 33.75 $\pm$ 1.58 & 39.95 $\pm$4.37 & 59.90 $\pm$ 3.89 & 48.20 $\pm$ 3.68\\
  LS & 55.14 $\pm$ 3.15 & 41.25 $\pm$ 3.28 & 41.75 $\pm$ 1.71 & 32.33 $\pm$ 1.03 & 59.79 $\pm$ 2.72 & 66.12 $\pm$ 6.82 \\
  MCFS         & 54.95 $\pm$ 3.28 & 44.88 $\pm$ 3.72 & 50.75 $\pm$ 1.25 & 50.38 $\pm$ 2.25 & \underline{66.55 $\pm$ 3.11}& {77.43 $\pm$ 7.15} \\
  UDFS       & 29.60 $\pm$ 2.73 & 38.21 $\pm$ 3.02 & 40.78 $\pm$ 1.03 & \textbf{55.57} $\pm$ \textbf{2.92} & 50.59 $\pm$ 1.97 & 65.32 $\pm$ 6.18 \\
  RSR       & 49.88 $\pm$ 3.75 & 45.48 $\pm$ 3.34 & {53.24 $\pm$ 1.83} & 37.52 $\pm$ 2.23 & 62.54 $\pm$ 2.34 & 72.54 $\pm$ 6.52 \\
  NDFS    & 54.33 $\pm$ 3.73 & 45.79 $\pm$ 3.81 & 49.85 $\pm$ 1.69 & 34.10 $\pm$ 3.81 & 63.32 $\pm$ 3.35 & 67.80 $\pm$ 6.48\\
  GLSPFS & 54.09 $\pm$ 3.22 & 50.84 $\pm$ 5.34 & \underline{53.63 $\pm$ 2.62} & 45.94 $\pm$ 2.38 & 64.65 $\pm$ 3.69 & \underline{78.00 $\pm$ 7.47}\\
  MFFS      & \underline{55.39 $\pm$ 3.32} & 49.09 $\pm$ 3.64 & 50.19 $\pm$ 1.64 & 36.57 $\pm$ 2.32  & 63.30 $\pm$ 3.36  & 73.55 $\pm$ 7.68\\
  GLPSL & 49.05 $\pm$ 3.02 & \textbf{53.97} $\pm$ \textbf{3.45} & 41.72 $\pm$ 1.05 & 47.52 $\pm$ 1.87 & 51.91 $\pm$ 2.18 & 72.16 $\pm$ 7.03 \\
  GLoSS    & \textbf{62.45} $\pm$ \textbf{3.58} & \underline{53.45 $\pm$ 3.88} & \textbf{54.27} $\pm$ \textbf{1.87} & \underline{52.76 $\pm$ 2.12} & \textbf{67.24} $\pm$ \textbf{3.27} & \textbf{79.37} $\pm$ \textbf{7.34}\\
  \hline
\end{tabular}
\end{table}

\subsubsection{Compare the performance with all features}\label{sec:test2}
To illustrate the effect of feature selection to clustering, we compare the clustering results using all features and selected features given by different methods. Figure \ref{fig.ACC} plots the ACC value and Figure \ref{fig.NMI} the NMI value with respect to the number of selected features. The baseline corresponds to the results using all features. 
From the figures, we see that in most cases, the proposed GLoSS method gives the best results, and selecting reasonably many features (but far less than the total number of features), it can give comparable and even better clustering results than those by using all features. Hence, the feature selection eliminates the redundancy of the data for clustering purpose. In addition, note that using fewer features can save the clustering time of the $K$-means method, and thus feature selection can improve both clustering accuracy and efficiency. 

\begin{table}[htp]
\centering
\footnotesize
\caption{Clustering results (NMI\% $\pm$ std\%) of different feature selection algorithms on different datasets. The best results are
highlighted in \textbf{bold} and the second best results are underlined. (The higher NMI, the better result is.)}
\label{tab.NMI}
\begin{tabular}{ccccccc}
  \hline
  Dateset    & Isolet           & Yale64           & Orl64            & WarpPIE &        Usps               & Orlraw \\
  \hline
  PCA & 61.48 $\pm$ 1.20 & 41.43 $\pm$ 2.72 & 58.57 $\pm$ 0.86 & 42.83 $\pm$ 3.82 & 56.08 $\pm$ 1.54 & 57.30 $\pm$ 3.93\\
  LS & 69.73 $\pm$ 1.43 & 46.88 $\pm$ 2.07 & 62.61 $\pm$ 1.53 & 30.06 $\pm$ 2.89 & 56.62 $\pm$ 0.95 & 73.38 $\pm$ 3.12\\
  MCFS         & 69.82 $\pm$ 1.37 & 53.70 $\pm$ 1.58 & 69.33 $\pm$ 1.62 & 54.37 $\pm$ 4.95 & \underline{61.01 $\pm$ 0.92}& 83.91 $\pm$ 3.53 \\
  UDFS       & 44.98 $\pm$ 1.02 & 47.40 $\pm$ 1.64 & 62.38$\pm$ 1.41 & \underline{54.55 $\pm$ 4.38} & 41.31 $\pm$ 1.21 & 68.78 $\pm$ 3.45\\
  RSR       & 63.47 $\pm$ 1.10 & 56.08 $\pm$ 1.43 & 72.33 $\pm$ 1.75 & 41.81 $\pm$ 3.75 & 55.32 $\pm$ 1.52 & \underline{83.96 $\pm$ 4.35} \\
  NDFS   & 70.05 $\pm$ 2.00 & 54.67 $\pm$ 2.35 & 70.42 $\pm$ 1.14 & 28.16 $\pm$ 4.45 & 58.78 $\pm$ 0.99 & 78.81 $\pm$ 3.99\\
  GLSPFS & 68.80 $\pm$ 1.07 & 56.18 $\pm$ 3.40 & \textbf{73.05 $\pm$ 1.52} & 52.23 $\pm$ 4.42 & 60.33 $\pm$ 1.65 & 82.99 $\pm$ 4.73\\
  MFFS      & \underline{72.64 $\pm$ 1.73} & 56.17 $\pm$ 4.47 & 70.65 $\pm$ 1.25 & 40.95 $\pm$ 3.39 & 59.11 $\pm$ 0.76 & 81.09 $\pm$ 4.12 \\
  GLPSL       & 65.41 $\pm$ 1.23 & \textbf{61.39} $\pm$ \textbf{1.72} & 64.76 $\pm$ 1.50 & 53.33 $\pm$ 3.89 & 40.98 $\pm$ 0.87 & 72.97 $\pm$ 3.37 \\
  GLoSS    & \textbf{74.28} $\pm$ \textbf{1.25} & \underline{58.87 $\pm$ 1.65} & \underline{73.02 $\pm$ 2.02} & \textbf{55.76} $\pm$ \textbf{4.56} & \textbf{61.29} $\pm$ \textbf{1.25} & \textbf{85.65} $\pm$ \textbf{4.15}\\
  \hline
\end{tabular}
\end{table}

\begin{figure}[!ht]
            \centering
             \includegraphics[width=\textwidth]{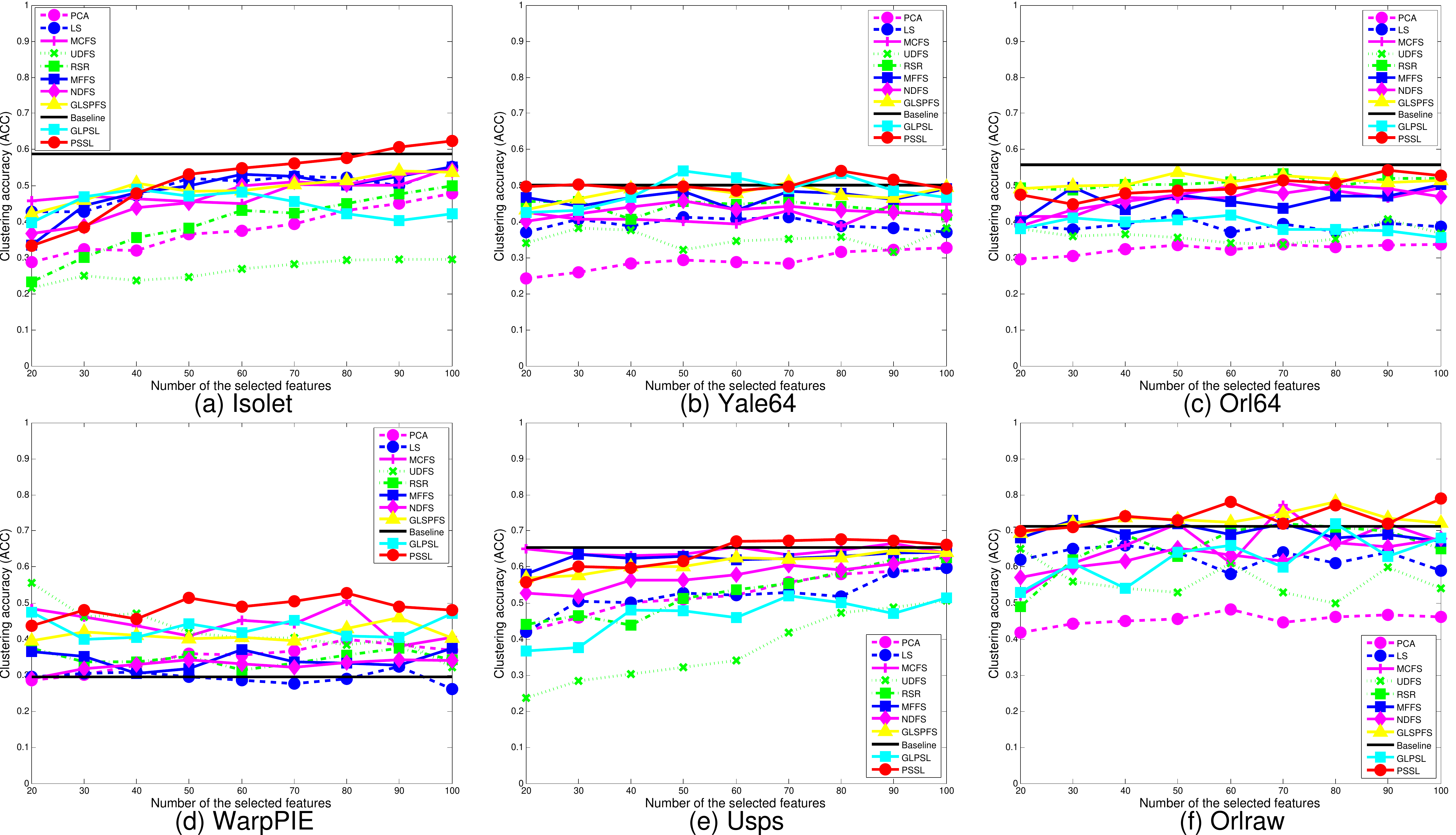}
            \caption{The clustering accuracy (ACC) of using all features and selected features by different methods.}
            \label{fig.ACC}
\end{figure}

\begin{figure}[!ht]
            \centering
            \centering
             \includegraphics[width=\textwidth]{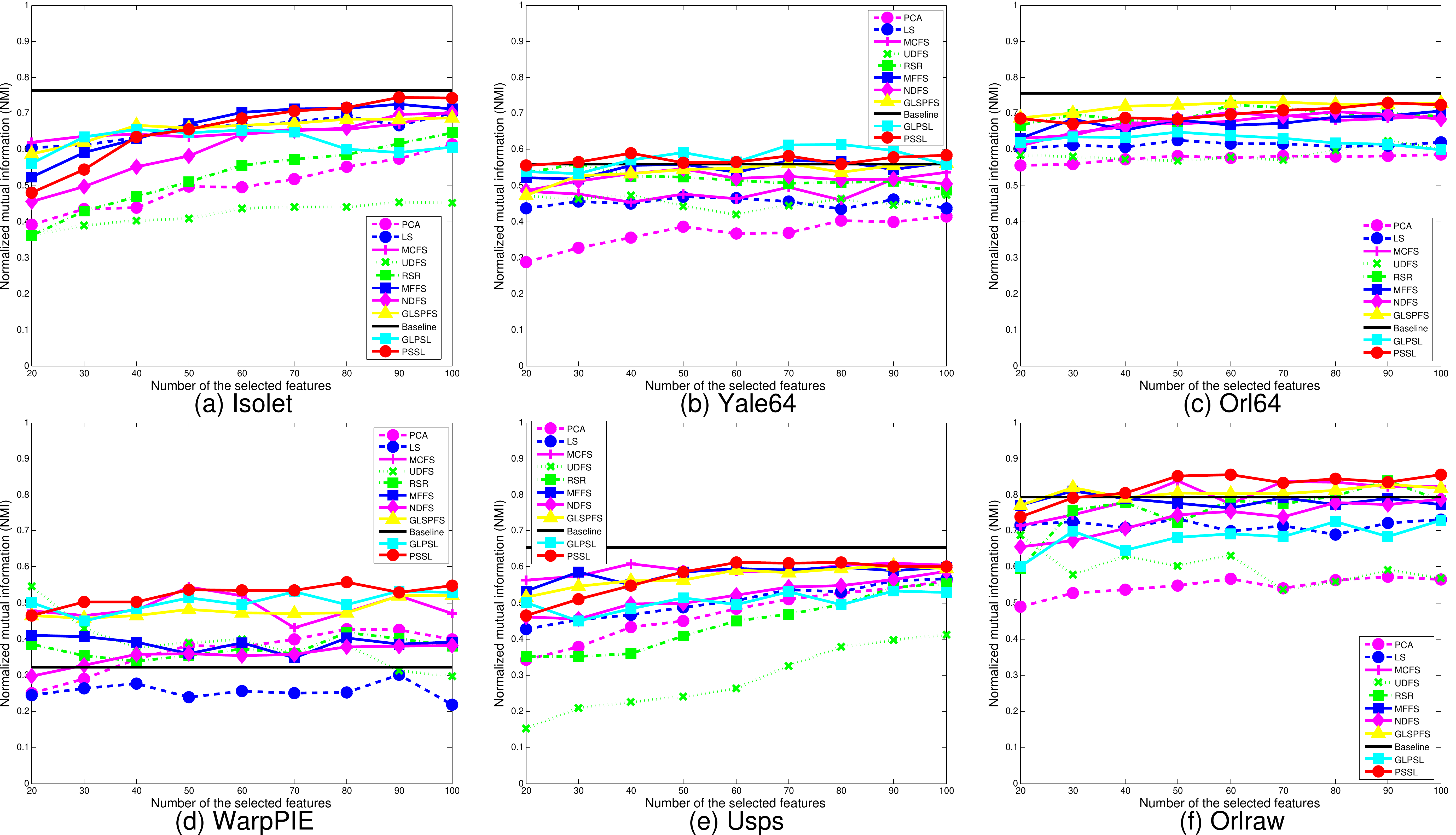}
            \caption{The normalized mutual information (NMI) of using all features and selected features by different methods.}
            \label{fig.NMI}
\end{figure}

\subsubsection{Sensitivity of parameters}\label{sec:test3}
To further demonstrate the performance of the proposed GLoSS method, we study its sensitivity with regard to the parameters $\kappa, \mu$ and $\beta$ in \eqref{eqn.LPSSL}. First, we fix $\mu=1$ and vary $\kappa$ and $\beta$. Figures \ref{fig.ACCbeta} and \ref{fig.NMIbeta} plot the ACC and NMI values given by GLoSS for different $\kappa$ and $\beta$'s. From the figures, we see that except for Isolet, GLoSS performs stably well for different combinations of $\kappa$ and $\beta$, and thus the users can choose the parameters within a large interval to have satisfactory clustering performance. Secondly, we fix $\beta=1$ and vary $\kappa$ and $\mu$. 
Figures \ref{fig.ACCmu} and \ref{fig.NMImu} plot the ACC and NMI values given by GLoSS for different $\kappa$ and $\mu$'s. Again, we see that GLoSS performs stably well except for the Isolet dataset.
\begin{figure}[!ht]
            \centering
             \includegraphics[width=\textwidth]{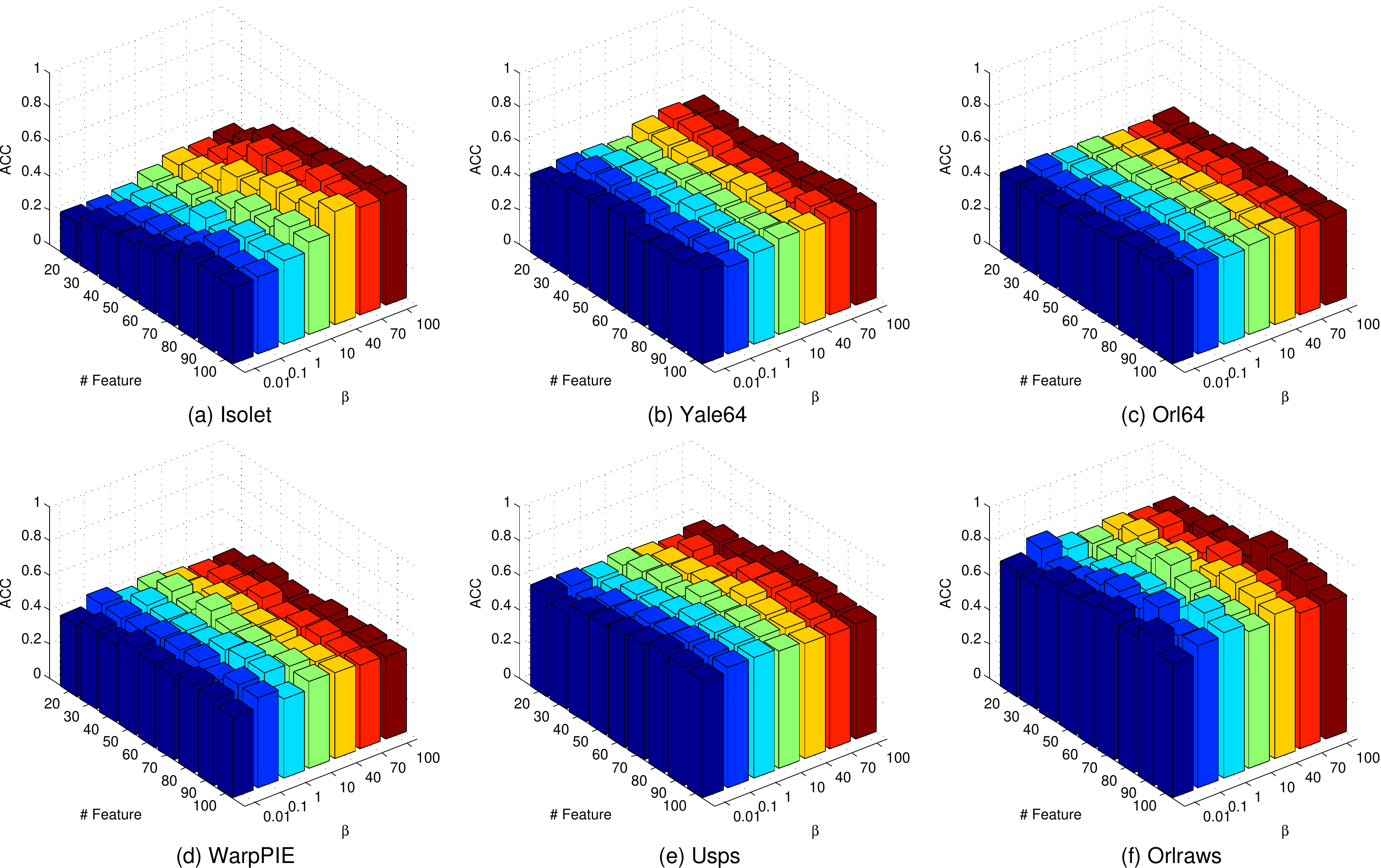}
            \caption{Clustering accuracy (ACC) produced by GLoSS with different $\kappa$ and $\beta$.}
            \label{fig.ACCbeta}
\end{figure}

\begin{figure}[!ht]
            \centering
            \centering
             \includegraphics[width=\textwidth]{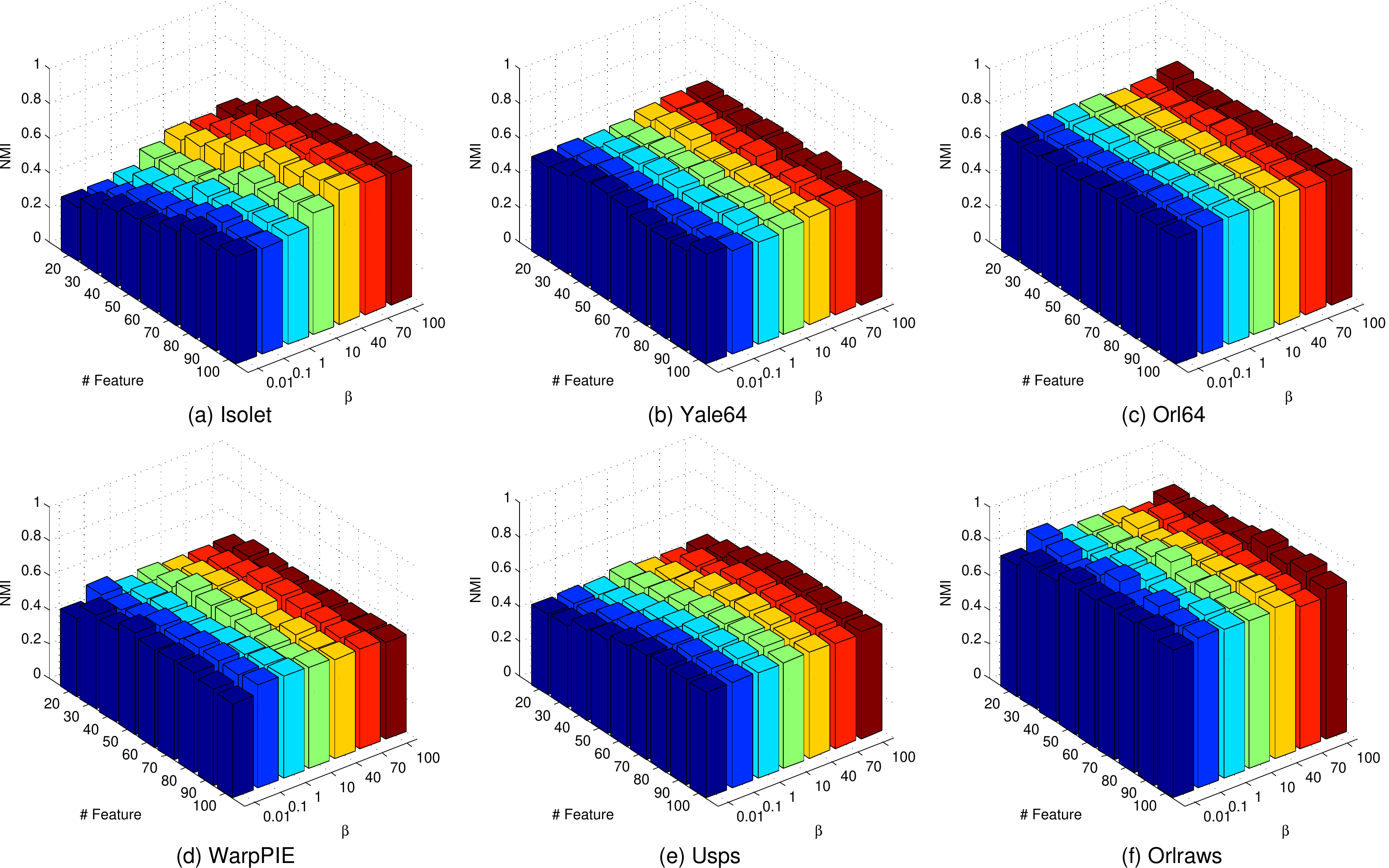}
            \caption{Normalized mutual information (NMI) produced by GLoSS with different $\kappa$ and $\beta$.}
            \label{fig.NMIbeta}
\end{figure}

\begin{figure}[!ht]
            \centering
             \includegraphics[width=\textwidth]{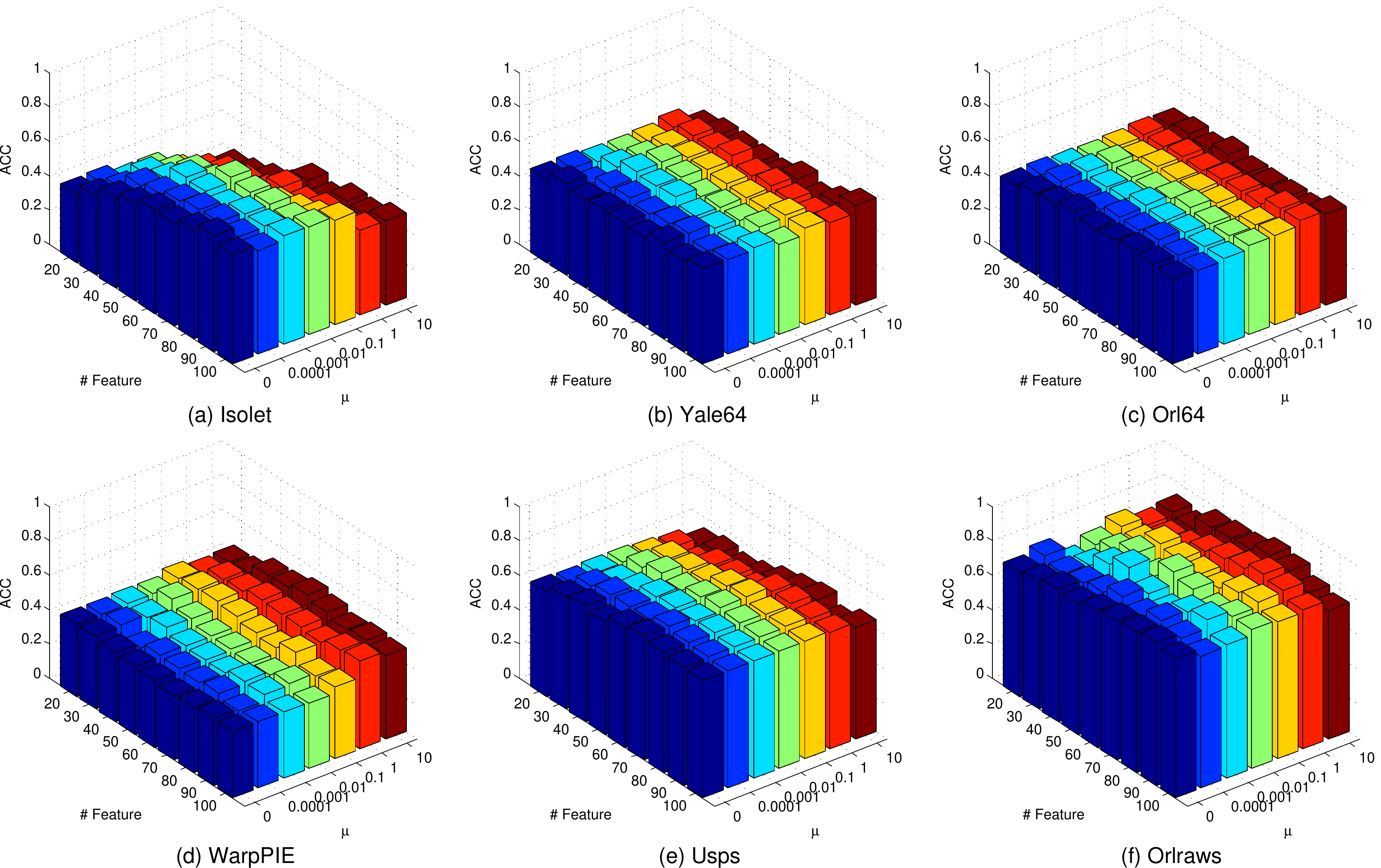}
            \caption{The clustering accuracy (ACC) given by GLoSS with different $\kappa$ and $\mu$.}
            \label{fig.ACCmu}
\end{figure}

\begin{figure}[!ht]
            \centering
            \centering
             \includegraphics[width=\textwidth]{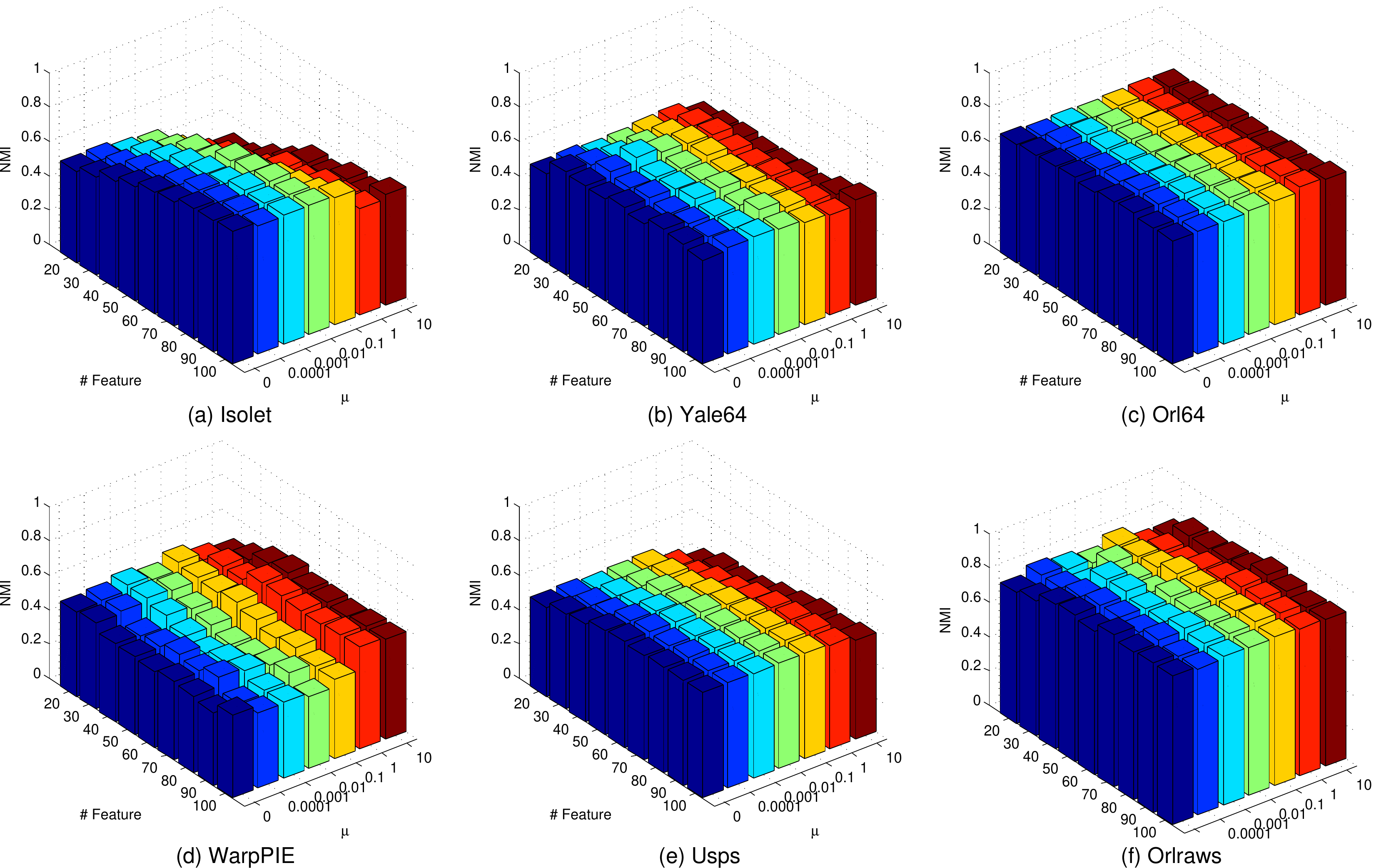}
            \caption{The normalized mutual information (NMI) given by GLoSS with different $\kappa$ and $\mu$.}
            \label{fig.NMImu}
\end{figure}

\section{Conclusions}\label{sec.conclusion}
We have proposed a new unsupervised joint model on subspace learning and feature selection. The model preserves both global and local structure of the data, and it is derived by relaxing an existing combinatorial model with 0-1 variables. A greedy algorithm has been developed, for the first time, to solve the combinatorial problem, and an accelerated block coordinate descent (BCD) method was applied to solve the relaxed continuous probelm. We have established the whole iterate sequence convergence of the BCD method. Extensive numerical tests on real-world data demonstrated that the proposed method outperformed several state-of-the-art unsupervised feature selection methods.

\section*{Acknowledgements}
The authors would like to thank two anonymous referees for their careful reviews and constructive comments.
Y. Xu is partially supported by AFOSR.
W. Pedrycz is partially supported by NSERC and CRC.

\appendix
\section{Proof of Theorem \ref{thm:subcvg}}
For simplicity, we assume $\omega_k=0,\,\forall k$, i.e., there is no extrapolation. The case of $\omega_k\not\equiv0$ is more complicated but can be treated similarly with more care taken to handle details; see \cite{xu2013block} for example.

The following result is well-known (c.f. Lemma 2.1 of \cite{xu2013block})
\begin{equation}\label{eq:diff-w}
F(W^{k},H^k)-F(W^{k+1},H^k)\ge\frac{L_w^k}{2}\|W^{k+1}-W^k\|_F^2\ge\frac{L_\mu}{2}\|W^{k+1}-W^k\|_F^2,
\end{equation}
where
\begin{equation}\label{eq:lip-mu}
L_\mu=\mu\|X^\top LX\|_2>0.
\end{equation}
By Lemma 3.1 of \cite{xu2014ihosvd}, we have
\begin{equation}\label{eq:diff-h}
\frac{1}{2}\|X-XW^{k+1}H^k\|_F^2-\frac{1}{2}\|X-XW^{k+1}H^{k+1}\|_F^2=\frac{1}{2}\|XW^{k+1}H^k-XW^{k+1}H^{k+1}\|_F^2
\end{equation}
and
\begin{equation}\label{eqh1}
XW^{k+1}H^k-XW^{k+1}H^{k+1}=U^{k+1}(U^{k+1})^\top\big(XW^{k+1}H^k-X\big),
\end{equation}
where $U^{k+1}$ contains the left $r$ leading singular vectors of $XW^{k+1}$ and $r$ is the rank of $XW^{k+1}$.

Note that $$F(W^{k+1},H^k)- F(W^{k+1},H^{k+1})=\frac{1}{2}\|X-XW^{k+1}H^k\|_F^2-\frac{1}{2}\|X-XW^{k+1}H^{k+1}\|_F^2.$$
Hence, summing \eqref{eq:diff-w} and \eqref{eq:diff-h} over $k$ and noting nonnegativity of $F$ we obtain
$$\sum_{k=0}^\infty\left(\frac{L_\mu}{2}\|W^{k+1}-W^k\|_F^2+\frac{1}{2}\|XW^{k+1}H^k-XW^{k+1}H^{k+1}\|_F^2\right)\le F(W^{0},H^0),$$
and thus
\begin{equation}\label{eq:lim-w}
\lim_{k\to\infty}W^{k+1}-W^k=0.
\end{equation}
and
\begin{equation}\label{eq:lim-h}
\lim_{k\to\infty}U^{k+1}(U^{k+1})^\top\big(XW^{k+1}H^k-X\big)=\lim_{k\to\infty}XW^{k+1}H^k-XW^{k+1}H^{k+1}=0.
\end{equation}
Combining the two equalities in \eqref{eq:lim-h}, we have
\begin{equation*}
\lim_{k\to\infty}U^{k}(U^{k})^\top\big(XW^{k}H^k-X\big)=0.
\end{equation*}
Since $\{XW^k\}$ is bounded and $(XW^k)^\top=(XW^k)^\top U^{k}(U^{k})^\top$, left multiplying $(XW^k)^\top$ in the above equation gives
\begin{equation}\label{eq:lim-h2}
\lim_{k\to\infty}(XW^k)^\top\big(XW^{k}H^k-X\big)=0.
\end{equation}

Assume $(\bar{W},\bar{H})$ is a finite limit point of $\{(W^k,H^k)\}_{k=1}^\infty$. Then there exists a subsequence $\{(W^k,H^k)\}_{k\in\mathcal{K}}$ convergent to $(\bar{W},\bar{H})$. If necessary, taking another subsequence, we can assume $L_w^k\to \bar{L}$ for some $\bar{L}>0$ as $\mathcal{K}\ni k\to\infty$. From \eqref{eq:lim-h2}, it holds that
$$\nabla_H f(\bar{W},\bar{H})=(X\bar{W})^\top(X\bar{W}\bar{H}-X)=0.$$
In addition, from the update rule of $W$, we have
$$W^{k+1}=\argmin_{W\ge 0}\langle \nabla_W f(W^k,H^k),W-W^k\rangle+\frac{L_w^k}{2}\|W-W^k\|_F^2+g_\beta(W).$$
Letting $\mathcal{K}\ni k\to\infty$ in the above equation and using \eqref{eq:lim-w} yield
$$\bar{W}=\argmin_{W\ge 0}\langle \nabla_W f(\bar{W},\bar{H}),W-\bar{W}\rangle+\frac{\bar{L}}{2}\|W-\bar{W}\|_F^2+g_\beta(W),$$
which implies
$$0\in \nabla_W f(\bar{W},\bar{H})+\partial g_\beta(W) +\partial \iota_+(\bar{W})=\partial_W F(\bar{W},\bar{H}).$$
Therefore, $(\bar{W},\bar{H})$ is a critical point of \eqref{eqn.LPSSL}.

\section{Proof of Theorem \ref{thm:glbcvg}}
For simplicity of notation, we let $Z^k=(W^k,H^k)$ and $\bar{Z}=(\bar{W},\bar{H})$. In addition, we assume $\omega_k=0,\forall k$ as in the proof of Theorem \ref{thm:subcvg}. Again, the case of $\omega_k\not\equiv 0$ can be shown similarly.
Let $\sigma_{\min}(X\bar{W})>0$ be the smallest singular value of $X\bar{W}$. By the continuity of singular value function and spectral norm of a matrix, there exists $\delta>0$ such that
\begin{subequations}\label{eq:conds}
\begin{align}
&\sigma_{\min}(XW)\ge \frac{\sigma_{\min}(X\bar{W})}{2},\text{ and }\|XW\|_2\le 2\|X\bar{W}\|_2,\,\forall W\in \mathcal{B}(\bar{W},\delta),\label{eq:cont-w}\\
&\|HH^\top\|_2\le 2\|\bar{H}\bar{H}^\top\|_2,\,\forall H\in \mathcal{B}(\bar{H},\delta),\label{eq:cont-h}
\end{align}
\end{subequations}
where $\sigma_{\min}(A)$ denotes the smallest singular value of matrix $A$, and $\mathcal{B}(\bar{A},\delta):=\{A:\,\|A-\bar{A}\|_F\le \delta\}$.

Since $F$ is a semi-algebraic function and continuous in its domain, it exhibits the so-called Kurdyka-{\L}ojasiewicz property (c.f. \cite{bolte2007lojasiewicz}): in a neighborhood $\mathcal{B}(\bar{Z},\rho)$, there exists $\phi(s)=cs^{1-\theta}$ for some $c>0$ and $0\le \theta <1$ such that
\begin{equation}\label{eq:KL-F}
\phi'(|F(Z)-F(\bar{Z})|)\text{dist}(0,\partial F(Z))\ge 1, \text{ for any }Z\in\mathcal{B}(\bar{Z},\rho)\cap\text{dom}(F)\text{ and }F(Z)\neq F(\bar{Z}).
\end{equation}
Let
$$F_k=F(Z^k)-F(\bar{Z}),\text{ and }\phi_k=\phi(F_k).$$
Without loss of generality, we assume $Z^0$ is sufficiently close to $\bar{Z}$ such that
\begin{equation}\label{eq:init}
2\|Z^0-\bar{Z}\|_F+3\left(\sqrt{\frac{2F_0}{L_\mu}}+\sqrt{\frac{8F_0}{\sigma_{\min}^2(X\bar{W})}}\right)+\frac{C_1}{2C_2}\phi_0<\rho,
\end{equation}
where $L_\mu$ is defined in \eqref{eq:lip-mu}, and
\begin{align}
&C_1=L_\delta + 2\|\bar{H}\bar{H}^\top\|_2\|XX^\top\|_2+L_\mu,\label{eq:c1}\\
&C_2=\frac{L_\mu}{2}+\frac{\sigma_{\min}^2(X\bar{W})}{8}.\label{eq:c2}
\end{align}
In the above equation, $L_\delta$ is the Lipschitz constant of $\nabla_W f(W,H)$ in $\mathcal{B}(\bar{Z},\delta)$, i.e.,
\begin{equation}\label{eq:liprho}
\|\nabla_W f(\hat{Z})-\nabla_W f(\tilde{Z})\|_F\le L_\delta \|\hat{Z}-\tilde{Z}\|_F,\,\forall \hat{Z},\tilde{Z}\in\mathcal{B}(\bar{Z},\delta).
\end{equation}
Note that $L_\delta$ must be finite since $f(W,H)$ is twice continuous differentiable and $\mathcal{B}(\bar{Z},\delta)$ is bounded. Otherwise if \eqref{eq:init} does not hold, since $\bar{Z}$ is a limit point of $\{Z^k\}$, we can take an iterate $Z^{k_0}$ sufficiently close to $\bar{Z}$ and let $Z^{k_0}$ be the new starting point. If neccessary, taking a smaller $\rho$, we assume
\begin{equation}\label{eq:cond-rho}
\rho+\sqrt{\frac{2F_0}{L_\mu}}\le \delta,
\end{equation}
where $\delta$ is the quantity in \eqref{eq:conds}.

From \eqref{eq:diff-w} and $F_{k+1}\le F_k \le F(\bar{Z}),\,\forall k$, we have
$\|W^1-W^0\|_F\le \sqrt{\frac{2F_0}{L_\mu}}$ and thus
\begin{equation}\label{eq:w1}\|W^1-\bar{W}\|_F\le\|W^1-W^0\|_F+\|W^0-\bar{W}\|_F\le \|W^0-\bar{W}\|_F+\sqrt{\frac{2F_0}{L_\mu}}<\rho\le \delta.
\end{equation}
Hence, $\sigma_{\min}(XW^1)\ge\frac{\sigma_{\min}(X\bar{W})}{2}$ from \eqref{eq:cont-w}, and
$$F(W^1,H^0)-F(W^1,H^1)\ge\frac{[\sigma_{\min}(XW^1)]^2}{2}\|H^1-H^0\|_F^2\ge\frac{[\sigma_{\min}(X\bar{W})]^2}{8}\|H^1-H^0\|_F^2,$$
which implies $\|H^1-H^0\|_F\le \sqrt{\frac{8F_0}{[\sigma_{\min}(X\bar{W})]^2}}$. Therefore,
\begin{equation}\label{eq:h1}
\|H^1-\bar{H}\|_F\le \|H^1-H^0\|_F+\|H^0-\bar{H}\|_F\le \|H^0-\bar{H}\|_F+ \sqrt{\frac{8F_0}{[\sigma_{\min}(X\bar{W})]^2}}.
\end{equation}
Combining \eqref{eq:w1} and \eqref{eq:h1}, we have
$$\|Z^1-\bar{Z}\|_F\le \|W^1-\bar{W}\|_F+\|H^1-\bar{H}\|_F\le 2\|Z^0-\bar{Z}\|_F+\sqrt{\frac{2F_0}{L_\mu}}+ \sqrt{\frac{8F_0}{\sigma_{\min}^2(X\bar{W})}},$$
which together with \eqref{eq:init} implies $Z^1\in \mathcal{B}(\bar{Z},\rho)$.

Assume that for some integer $K$, $Z^k\in\mathcal{B}(\bar{Z},\rho), \forall 0\le k\le K$. We go to show $Z^{K+1}\in\mathcal{B}(\bar{Z},\rho)$ and thus by induction $Z^k\in\mathcal{B}(\bar{Z},\rho),\,\forall k$. Note that
\begin{align*}
&0\in\nabla_W f(W^{k-1},H^{k-1}) + L_w^{k-1}(W^{k}-W^{k-1})+\partial g_\beta(W^{k})+\partial\iota_+(W^{k}),\\
&0=\nabla_H f(W^{k},H^{k}).
\end{align*}
Hence,
\begin{align}
\text{dist}(0,\partial F(Z^k))\le & \|\nabla_W f(W^{k},H^{k})-\nabla_W f(W^{k-1},H^{k-1})\|_F+L_w^{k-1}\|W^{k}-W^{k-1}\|_F\cr
\le & C_1 \|Z^k-Z^{k-1}\|_F,\label{eq:dist}
\end{align}
where $C_1$ is defined in \eqref{eq:c1}. 
In addition, we have
\begin{align}
& \phi_k-\phi_{k+1}\cr
\ge & \phi'(F_k)(F_k-F_{k+1})\quad(\text{ from concavity of }\phi)\cr
\ge & \frac{F_k-F_{k+1}}{C_1 \|Z^k-Z^{k-1}\|_F}\quad(\text{ from KL property }\eqref{eq:KL-F})\cr
\ge & \frac{C_2 \|Z^{k+1}-Z^{k}\|_F^2}{C_1 \|Z^k-Z^{k-1}\|_F},\label{eq:sqrt}
\end{align}
where the last inequality follows from \eqref{eq:c2}, \eqref{eq:diff-w} and
$$F(W^{k+1},H^k)-F(W^{k+1},H^{k+1})\ge \frac{\sigma_{\min}^2(X\bar{W})}{8}\|H^k-H^{k+1}\|_F^2.$$
Transforming \eqref{eq:sqrt} gives
\begin{align*}
&C_2 \|Z^{k+1}-Z^{k}\|_F^2\le C_1 \|Z^k-Z^{k-1}\|_F(\phi_k-\phi_{k+1})\\
\Rightarrow &\sqrt{C_2} \|Z^{k+1}-Z^{k}\|_F\le\sqrt{C_1 \|Z^k-Z^{k-1}\|_F(\phi_k-\phi_{k+1})}\\
\Rightarrow & \sqrt{C_2} \|Z^{k+1}-Z^{k}\|_F\le\frac{\sqrt{C_2}}{2}\|Z^k-Z^{k-1}\|_F+\frac{C_1}{2\sqrt{C_2}}(\phi_k-\phi_{k+1}).
\end{align*}
Summing the above inequality over $k$ and arranging terms give
\begin{align}\label{eq:sum-K}
\sum_{k=1}^K\|Z^{k+1}-Z^k\|_F\le\|Z^1-Z^0\|_F+\frac{C_1}{2C_2}(\phi_1-\phi_{K+1}).
\end{align}
Hence,
\begin{align}
\|Z^{K+1}-\bar{Z}\|_F\le & \sum_{k=1}^K\|Z^{k+1}-Z^k\|_F+\|Z^1-\bar{Z}\|_F\cr
\le & \|Z^1-\bar{Z}\|_F+\|Z^1-Z^0\|_F+\frac{C_1}{2C_2}\phi_0\cr
\le & 2\|Z^1-Z^0\|_F+\|Z^0-\bar{Z}\|_F+\frac{C_1}{2C_2}\phi_0\cr
(\text{from }\eqref{eq:w1}\text{ and }\eqref{eq:h1})\quad\le &\, \rho,
\end{align}
which indicates $Z^{K+1}\in\mathcal{B}(\bar{Z},\rho)$. By induction, we have $Z^k\in\mathcal{B}(\bar{Z},\rho),\,\forall k,$ and thus \eqref{eq:sum-K} holds for all $K$. Therefore, $\{Z^k\}_{k=1}^\infty$ is a Cauchy sequence and converges. Since $\bar{Z}$ is a limit point, it must hold that $\lim_{k\to\infty}Z^k=\bar{Z}$. This completes the proof.

\section*{References}




\bibliographystyle{elsarticle-num}
\bibliography{bibLPSSLFS}

\end{document}